\documentclass[11pt]{article}
\usepackage[letterpaper, margin=1in]{geometry}

\usepackage{tikz}
\usetikzlibrary{positioning, shapes, fit, backgrounds, arrows.meta, calc}

\usepackage{microtype}
\usepackage{graphicx}
\usepackage{subcaption}
\usepackage{booktabs} 
\usepackage{hyperref}

\usepackage{amsmath}
\usepackage{amssymb}
\usepackage{mathtools}
\usepackage{amsthm}

\usepackage{algorithm}
\usepackage{algorithmic}

\usepackage[capitalize,noabbrev]{cleveref}
\usepackage{natbib}

\usepackage{dsfont}
\usepackage[mathscr]{euscript}
\usepackage{xcolor}
\usepackage{colortbl}
\usepackage{thmtools}
\usepackage{thm-restate}

\usepackage{multirow}
\usepackage{wrapfig}

\usepackage{pifont}

\usepackage{MnSymbol}
\DeclareMathAlphabet\mathbb{U}{msb}{m}{n}
\usepackage{xpatch}

\def\Rset{\mathbb{R}}

\let\Pr\undefined

\DeclareMathOperator*{\Pr}{\mathbb{P}}

\DeclareMathOperator*{\E}{\mathbb E}

\DeclareMathOperator*{\argmin}{argmin}

\DeclareMathOperator{\Reg}{\mathsf{Reg}}

\DeclareMathOperator{\supp}{supp}

\DeclarePairedDelimiter{\paren}{(}{)}
\DeclarePairedDelimiter{\norm}{\|}{\|}
\DeclarePairedDelimiter{\tri}{\langle}{\rangle}

\ExplSyntaxOn
\tl_const:Nn \c_my_uc_alphabet_tl { 
ABCDEFGHIJKLMNOPQRSTUVWXYZ }
\tl_const:Nn \c_my_full_alphabet_tl { ABCDEFGHIJKLMNOPQRSTUVWXYZ
  abcdefghijklmnopqrstuvwxyz }

\tl_map_inline:Nn \c_my_uc_alphabet_tl
 { \cs_gset:cpn { c#1 } { \mathcal{#1} } }

\tl_map_inline:Nn \c_my_uc_alphabet_tl
 { \cs_gset:cpn { s#1 } { \mathscr{#1} } }

\tl_map_inline:Nn \c_my_full_alphabet_tl
 {
  \str_if_eq:nnTF { #1 } { f }
    { 
    }
    {
      \cs_gset:cpn { b#1 } { \mathbf{#1} }
    }
  
  \cs_gset:cpn { sf#1 } { \mathsf{#1} }
 }
\ExplSyntaxOff

\newcommand{\btheta}{{\boldsymbol \theta}}

\newcommand{\h}{\widehat}

\newcommand{\wt}{\widetilde}

\newcommand{\ignore}[1]{}

\hypersetup{
  breaklinks   = true, 
  colorlinks   = true, 
  urlcolor     = blue, 
  linkcolor    = blue, 
  citecolor    = blue 
}

\usepackage[toc, page, header]{appendix}
\declaretheorem{theorem}
\newtheorem{lemma}[theorem]{Lemma}

\newtheorem{corollary}[theorem]{Corollary}
\newtheorem{definition}[theorem]{Definition}

\title{Beyond Binary: Continuous State Optimization \\with Graph-Structured Objectives}
\author{
  Corinna Cortes\\
  Google Research\\
  \texttt{corinna@google.com}
  \and
  Yishay Mansour\\
  Tel Aviv University \\
  \& Google Research\\
  \texttt{mansour.yishay@gmail.com}
  \and
  Mehryar Mohri\\
  Google Research\\
  \& Courant Institute\\
  \texttt{mohri@google.com}
}

\begin{document}

\maketitle

\begin{abstract}
  Large-scale learning systems often face the challenge of balancing multiple,
  potentially competing objectives, such as fairness, accuracy, and
  latency. While recent work has formalized this as an optimization problem over
  binary states, many real-world control parameters, such as fairness
  thresholds, diversity mixing rates, or resource budgets, are continuous. In
  this work, we extend the framework to \emph{continuous state spaces}. We model
  the problem as minimizing a sum of linear objectives subject to \emph{movement
    costs} that penalize system instability. We capture the local structure of
  the objectives using a \emph{dependency graph} (or factor graph), where each
  objective is determined by a subset of the state attributes. To address the
  tension between exploration and stability, we propose \emph{Lazy
    Graph-LinUCB}, an algorithm that performs lazy updates to minimize switching
  costs while maintaining near-optimal regret. Beyond stability, we introduce
  three advanced mechanisms to exploit the underlying graph structure: (1) an
  \emph{asynchronous} update schedule that eliminates synchronization overhead
  in sparse graphs; (2) an \emph{adaptive} algorithm that learns the graph
  structure from data; and (3) a \emph{joint estimator} that leverages data
  sharing among correlated objectives to significantly tighten regret
  bounds. Empirically, we demonstrate that these structural exploitations reduce
  movement costs by more than a factor of three in heterogeneous systems while
  maintaining similar cumulative losses.
\end{abstract}

\section{Introduction}
\label{sec:intro}

Modern machine learning systems are rarely optimized for a single objective. A
recommender system, for instance, must balance relevance against diversity,
revenue, latency, and various fairness constraints across demographic
groups. These objectives are often inherently competing; increasing the
diversity of recommendations may degrade relevance, while enforcing strict
fairness constraints may increase latency or reduce revenue. Addressing these
conflicts requires a framework for navigating the trade-offs in a data-driven
manner. Recent work by
\citet*{AwasthiCortesMansourMohri2024,AwasthiCortesMansourMohri2026} proposed a
model for optimizing such systems based on user feedback (complaints).  In their
model, the system configuration is represented as a state in a Markov Decision
Process (MDP), and the goal is to find a configuration that minimizes the
aggregate cost-weighted volume of complaints. A key structural insight of their
work was the use of a \emph{dependency graph}, encoding which criteria
are mutually incompatible, and \emph{correlation sets} to model local
dependencies between criteria, allowing
for efficient learning even when the number of criteria is large. However, a
significant limitation of this prior work is the assumption that system states
are \emph{binary}. Each criterion is modeled as either \emph{fixed} (satisfied) or
\emph{unfixed} (violated). While this simplifies the analysis, it fails to capture
the nuance of real-world hyperparameters. For example, fairness constraints
typically involve a threshold $\tau \in [0, 1]$ (e.g., the allowable gap in True
Positive Rates), where a strict setting yields different feedback dynamics than
a loose one. Similarly, diversity mixing rates and resource allocation budgets
(e.g., latency targets) are inherently continuous parameters rather than binary
switches.

In this paper, we propose a natural and powerful extension of the competing
objectives framework to \emph{continuous state spaces}.  We replace the binary
hypercube $\{0, 1\}^k$ with the continuous domain $[0, 1]^k$. To maintain
tractability, we generalize the concept of \emph{Correlation Sets} to
\emph{Graph-Structured Linear Function Approximation}. We assume the loss
function can be approximated by a linear model $\phi(\bs)^\top \btheta$, where
the feature map $\phi(\bs)$ respects the sparsity structure of the underlying
dependency graph. Furthermore, we address a critical aspect of continuous
control: stability. In the binary setting, switching a criterion incurred a
fixed \emph{fixing cost}. In the continuous setting, we model this as a
\emph{movement cost} proportional to the magnitude of the change
$\|\bs_t - \bs_{t-1}\|$. This introduces a non-trivial trade-off between
\emph{exploration} (moving the state to disjoint regions to learn the loss
landscape) and \emph{stability} (minimizing the cost of adjustment). Our work
connects the literature on multi-objective optimization with \emph{Linear
  Bandits} \citep{abbasi2011improved} and \emph{Bandits with Switching Costs}
\citep{cesa2013online, dekel2014bandits, arora2019bandits, amir2022better},
offering a unified theoretical framework for optimizing complex, continuous
parameter spaces under competing feedback. While our algorithms build on
established foundations—specifically the OFUL principle (Optimism in the Face of
Uncertainty for Linear bandits) \citep{abbasi2011improved} for bandits and Online
Gradient Descent for adversarial settings—our primary algorithmic contribution
lies in adapting these mechanisms to the graph-structured and
sparsity-constrained nature of the competing objectives problem.  In the
stochastic setting, \textsc{LazyGraphLinUCB} departs from standard lazy bandit
algorithms by leveraging a decomposed determinant trigger. While the base
algorithm uses synchronized updates, we also introduce an asynchronous variant
that monitors local information gain across the dependency graph, ensuring that
stability in one objective does not unnecessarily hinder exploration in disjoint
graph neighborhoods. In the adversarial setting, we re-purpose the \emph{virtual
iterate} technique not to minimize movement costs (which standard OGD already
handles), but to explicitly satisfy operational sparsity budgets, quantifying
the precise regret trade-off required to maintain a low-frequency update
schedule.

\textbf{Contributions.}
We make four primary contributions. First, we extend the feedback-driven
competing objectives framework from binary system configurations to
\emph{continuous state spaces}, modeling the problem as a graph-structured
linear bandit with movement costs. This allows us to capture realistic control
parameters such as fairness thresholds and resource budgets while rigorously
accounting for stability.  Second, we propose \textsc{LazyGraphLinUCB}, a
stochastic algorithm that uses a decomposed determinant trigger to minimize
switching costs while maintaining optimal prediction regret.  Third, we
introduce three advanced mechanisms to exploit the underlying graph structure:
an \emph{asynchronous} update schedule that eliminates synchronization penalties
in sparse graphs; an \emph{adaptive} routine that learns dependencies from data;
and a \emph{joint estimator} based on a Factor Graph decomposition that tightens
regret bounds in dense, correlated systems. Finally, we provide matching minimax
lower bounds and empirically demonstrate that our structural exploitations
reduce movement costs by more than a factor of three in heterogeneous graph
systems, and up to a factor of five in real-world single-objective tasks.

The rest of the paper is organized as follows. Section~\ref{sec:related}
discusses related work. Section~\ref{sec:setup} presents our continuous-state
formulation. Section~\ref{sec:algorithm} introduces the \textsc{LazyGraphLinUCB}
algorithm and its regret analysis.  Section~\ref{sec:exploit_graph_structure}
presents advanced strategies for exploiting graph structure, including
asynchronous updates and adaptive learning. Section~\ref{sec:adversarial}
outlines the extension to the adversarial setting. Section~\ref{sec:experiments}
provides numerical illustrations to validate the theoretical findings.

\section{Related Work}
\label{sec:related}

Our work sits at the intersection of three lines of research: competing
objectives optimization, linear bandits, and online optimization with movement
costs.

\textbf{Competing objectives and feedback-driven optimization.}
\citet*{AwasthiCortesMansourMohri2024,AwasthiCortesMansourMohri2026} introduced
a feedback-driven framework for optimizing multiple competing objectives,
modeled as an MDP over binary states using incompatibility graphs and
correlation sets. Our work extends this framework to \emph{continuous} state
spaces, replacing binary fixes with smooth parameter tuning. This extension is
motivated by the vast literature on multi-objective optimization
\citep{sener2018multi, shah2016pareto, marler2004survey} and by algorithmic
fairness, where criteria such as equal opportunity and calibration are
inherently in tension \citep{kleinberg2017, hardt2016equality} and fairness
thresholds are continuous parameters \citep{agarwal2018reductions,
  cotter2018training}. Unlike prior work on fairness in bandits
\citep{joseph2016fairness, gillen2018online,
  LiuRadanovicDimitrakakisMandalParkes2017} and on agnostic multi-objective
algorithms \citep{cortes2020agnostic}, our framework addresses stability
explicitly via movement costs and exploits graph sparsity to scale to many
criteria.

\textbf{Linear bandits.}  Our stochastic analysis builds on the OFUL framework
of \citet{abbasi2011improved} and the contextual bandit work of
\citet{li2010contextual}. Graph-structured bandits have been studied in the
feedback-graph setting \citep{alon2015online, arora2019bandits}, where the graph
determines which arm rewards are observed. Our dependency graph plays a
different, complementary role: it defines the \emph{feature structure} of the
loss function, enabling decomposed estimation (akin to spectral bandits
\citep{valko2014spectral} or high-dimensional bandit methods with compatibility
conditions \citep{bastani2020online}) and localized policy updates.

\textbf{Bandits and online learning with switching costs.}  The problem of
minimizing regret while limiting the frequency of state changes has been studied
for adversarial and stochastic bandits \citep{dekel2014bandits, amir2022better,
  cesa2013online, sherman2021lazy}. These works bound the \emph{number} of
switches; our movement cost instead penalizes the \emph{magnitude} of each
change, capturing the physical reality of continuous-parameter adjustment. In
the online convex optimization literature, smoothed or lazy variants have been
studied for switching costs \citep{chen2018smoothed, zhang2021continuous,
  goel2019online, andrew2013tale, kalai2005efficient}. Our Randomized Lazy OGD
(Section~\ref{sec:adversarial}) connects to this line of work, while our
stochastic \textsc{LazyGraphLinUCB} achieves \emph{data-dependent} laziness via
the determinant-doubling trigger, avoiding the fixed-probability randomization
required in adversarial settings.

An extended discussion of related work, covering the broader fairness and
multi-criteria optimization literature in greater depth, is given in
Appendix~\ref{app:extended_related}.

\section{Problem Setup}
\label{sec:setup}

We model the problem as an online optimization game over $T$ rounds.  At each
round $t$, the learner selects a continuous state vector $\bs_t \in [0, 1]^k$.

\textbf{Notation.} We denote the set of criteria (nodes) by
$\cV = \{1, \dots, k\}$. The system state is a vector $\bs \in \sS = [0,
1]^k$. We assume an underlying dependency graph $\cG = (\cV, E)$, where $\cN(i)$
denotes the neighborhood of $i$; see Figure~\ref{fig:graph_structure} for an example.

\textbf{Contrast with Binary MDP Models.}
It is important to distinguish our formulation from the binary state model of
\citet{AwasthiCortesMansourMohri2024}. In their work, the system is modeled as
an MDP where explicit ``fixing actions'' induce probabilistic transitions between
satisfied and unsatisfied states. In contrast, our continuous framework adopts a
``direct control'' perspective akin to Bandit Convex Optimization. We assume the
learner can directly select any configuration $\bs_t \in [0, 1]^k$, but this
control comes at the price of a \emph{movement cost} that penalizes
instability. This shift allows us to handle the infinite cardinality of the
continuous domain while rigorously modeling the ``cost of change'' that was
previously captured by atomic fixing actions. One could argue that the Bandit
model is actually more realistic than an MDP: If an engineer sets a threshold to
$0.7$, the system state for that criterion becomes $0.7$. The uncertainty is in
the response (user complaints), not in the actuation (whether the knob turns).
Since the state transitions are deterministic ($\bs_{t+1} = \bs_t$), the
complexity of reinforcement learning (planning over long horizons) vanishes. The
core difficulty here is not reaching a state, but identifying the optimal state
under uncertainty and movement constraints, which is exactly the domain of
bandits with switching costs.

\textbf{Background: LinUCB.} Our approach builds on the OFUL principle
\citep{abbasi2011improved}. For a standard linear bandit, one maintains a
regularized least-squares estimator
$\h \theta_t = \bV_t^{-1} \sum_{\tau=1}^t x_\tau y_\tau$ with covariance
$\bV_t = \lambda_{\mathrm{reg}} \bI + \sum_{\tau=1}^t x_\tau x_\tau^\top$, where
$\lambda_{\mathrm{reg}}$ is a positive initialization hyperparameter. The
confidence set is
$\cC_t = \{ \theta : \| \theta - \h \theta_t \|_{\bV_t} \le \beta_t \}$, where
$\beta_t = \wt O(\sqrt{d \log t})$ is the confidence radius (see
Lemma~\ref{lemma:confidence-ellipsoid}). We adapt this framework to the
multi-objective setting by maintaining $k$ local estimators.

\subsection{Graph-structured loss}

At round $t$, the learner selects $\bs_t$ and observes a loss vector
$\ell_t(\bs_t) \in \Rset^k$.  Let $\ell_t(\bs_t) \in \Rset^k$ be the vector of
observed losses at time $t$, where $\ell_{t,i}(\bs_t)$ is the loss associated
with criterion $i$. We assume the expected loss
$\mu_i(\bs) = \E[\ell_{t,i}(\bs)]$ admits a linear representation, but with
specific structural constraints imposed by the graph $\cG$.

\begin{definition}[Graph-Structured Features]
  \label{def:graph_struct}
  We say the loss is \emph{graph-structured} with respect to $\cG$ if the
  expected loss for criterion $i$ depends only on the state of its local
  neighborhood $\cN(i)$. Formally, there exist local feature maps
  $\phi_i\colon [0, 1]^{|\cN(i)|} \to \Rset^{d_i}$ and parameters
  $\btheta^*_i \in \Rset^{d_i}$ such that:
  \begin{equation}
    \label{eq:1}
    \mu_i(\bs) = \tri*{ \phi_i(\bs|_{\cN(i)}), \btheta^*_i }.
  \end{equation}
\end{definition}
This formulation aligns with spectral bandits where rewards are smooth over a
graph \citep{valko2014spectral}, though we rely on local sparsity rather than
global smoothness.
Here, $\bs|_{\cN(i)}$ denotes the restriction of the state vector $\bs$ to the
indices in $\cN(i)$. This assumption generalizes the ``Correlation Sets'' of
\citet{AwasthiCortesMansourMohri2024}. It implies that the loss for criterion
$i$ is insensitive to changes in criteria $j \notin \cN(i)$, effectively
reducing the dimensionality of the learning problem related to criterion $i$
from $k$ to the size of its local neighborhood (see
Figure~\ref{fig:graph_structure}). The total expected complaint loss at state
$\bs$ is the sum of individual criteria losses:
\begin{equation}
  L(\bs) = \sum_{i=1}^k \mu_i(\bs)
  = \sum_{i=1}^k \langle \phi_i(\bs|_{\cN(i)}), \btheta^*_i \rangle.
\end{equation}

We will assume that for all criteria $i$ and state $\bs \in \sS$, the feature
vectors are bounded in Euclidean norm: $\|\phi_i(\bs)\|_2 \le L$. Without loss
of generality, we assume features are scaled such that $L \le 1$ and
$\lambda_{\mathrm{reg}} \ge 1$, which implies
$\|\phi_i(\bs)\|_{\bV_{t, i}^{-1}} \le 1$ as required for the analysis.

\textbf{Social Welfare Interpretation.}
We can view this formulation as a cooperative multi-agent game (see, e.g.,
\citep{gentile2014online} for related online multi-agent learning). Each
criterion $i \in \cV$ acts as an agent trying to minimize its own loss
$\mu_i(\bs)$, which depends on the actions of its neighbors $\cN(i)$. The global
objective $L(\bs)$ then corresponds to maximizing the social welfare (minimizing
the total loss) of the system.

\textbf{Correlated Noise Structure.}
While our analysis assumes for simplicity that the noise terms are independent
across criteria, the graph structure implies a natural correlated noise
model. If we view noise as arising from latent attributes associated with each
node $j$ (e.g., user-specific variance), the realized noise for criterion $i$
can be modeled as $\eta_{t,i} = \sum_{j \in \cN(i)} \xi_{t,j}$, where
$\xi_{t,j}$ are independent noise components. Crucially, since the sum of
independent sub-Gaussian variables is itself sub-Gaussian, our regret analysis
remains valid under this model (with the variance proxy $\sigma^2$ scaling with
the local degree $\Delta$). Furthermore, such correlations would only strengthen
the case for the Joint Estimator (Section~\ref{sec:data_sharing}), which could
be extended to exploit the resulting non-diagonal noise covariance via
Generalized Least Squares.

\subsection{Movement costs}

In continuous control systems, rapid changes to parameters can be disruptive.
Thus, unlike standard bandits, changing the state is costly.  We model this via
a \emph{movement cost} (or switching cost) that penalizes the distance between
consecutive states.
\begin{equation}
  C_{\mathrm{move}}(\bs_{t-1}, \bs_t) = \lambda \| \bs_t - \bs_{t-1} \|_1,
\end{equation}
where $\lambda > 0$ is a regularization parameter controlling the penalty for
instability. We choose the $\ell_1$-norm to model the aggregate adjustment
effort, though $\ell_2$ or other norms are possible.

\subsection{Objective: regret minimization}

The goal of the learner is to minimize the cumulative regret, defined as the
difference between the learner's total cost (complaints + movement) and the cost
of the optimal \emph{static} state $\bs^*$.  Let
$\bs^* = \argmin_{\bs \in \sS} L(\bs)$. The regret over horizon $T$ is:
\begin{equation}
  \Reg_T = \sum_{t=1}^T \left( L(\bs_t) - L(\bs^*) \right)
  + \sum_{t=1}^T \lambda \| \bs_t - \bs_{t-1} \|_1.
\end{equation}
Note that the comparator is the optimal static state, which of course incurs
zero movement cost.

\subsection{Concrete examples of graph-structured bases}
\label{sec:examples}

To illustrate the power of this formulation, we consider two specific instances
of the feature map $\phi_i$.

\textbf{Example 1: Pairwise Quadratic Interactions.}
Consider a case where the loss for criterion $i$ depends quadratically on its
own setting and the settings of its neighbors. For $j \in \cN(i)$, let $s_j$
denote the $j$-th component of $\bs$. We can define the basis:
\begin{equation}
  \phi_i(\bs)
  = \left[ 1, s_i, s_i^2, \{s_j\}_{j \in \cN(i) \setminus \{i\}},
    \{s_i s_j\}_{j \in \cN(i) \setminus \{i\}} \right]^\top.
\end{equation}
This captures self-convexity (via $s_i^2$) and pairwise interference (via
$s_i s_j$), modeling scenarios where increasing the threshold of neighbor $j$
exacerbates the loss of $i$.

\textbf{Example 2: Radial Basis Functions (RBF).}
For highly non-linear dependencies, we can use kernel approximation. Let
$\bk_i(\cdot, \cdot)$ be a kernel defined over the subspace $[0,
1]^{|\cN(i)|}$. The feature map $\phi_i$ corresponds to the Random Fourier
Features (RFF) of this kernel. This allows the algorithm to learn arbitrary
smooth loss surfaces, provided they
respect the local graph structure.

\textbf{Connection to Feedback Graphs.}
Our formulation shares conceptual roots with the literature on \emph{Bandits
  with Feedback Graphs} \citep{alon2015online, arora2019bandits}. In that
setting, a graph defines which arms' rewards are observed when a specific arm is
played. In our continuous linear setting, the dependency graph $\cG$ plays an
analogous role by defining the \emph{information flow}: observing the loss
components at state $\bs_t$ updates the confidence ellipsoids only for the local
parameters $\{\theta^*_i\}_{i=1}^k$ associated with the active
neighborhoods. Unlike the standard feedback graph setting where information is
discrete (observing a neighbor's reward), our \emph{feedback} is algebraic:
measurements propagate information to other states via the shared linear
structure of the local feature maps.

\ignore{ \setlength{\intextsep}{-1pt} \setlength{\columnsep}{10pt}
  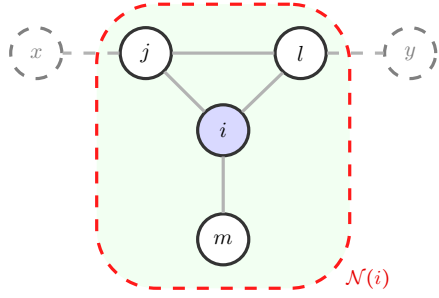
\begin{wrapfigure}{r}{0.4\columnwidth}
    \centering \resizebox{.2\textwidth}{!}{\begin{tikzpicture}[
    node distance=1.5cm,
    scale=0.85, transform shape,
    criterion/.style={circle, draw=black!80, thick, minimum size=0.9cm, fill=white, font=\bfseries},
    connection/.style={thick, draw=gray!60},
    highlight/.style={draw=red!90, fill=green!5, dashed, rounded corners=15pt, thick}
]

    \node[criterion, fill=blue!15] (i) at (0,0) {$i$};
    
    \node[criterion] (j) [above left=of i] {$j$};
    \node[criterion] (l) [above right=of i] {$l$};
    \node[criterion] (m) [below=of i] {$m$};
    
    \node[criterion, dashed, draw=gray, text=gray] (x) [left=of j] {$x$};
    \node[criterion, dashed, draw=gray, text=gray] (y) [right=of l] {$y$};

    \draw[connection] (i) -- (j);
    \draw[connection] (i) -- (l);
    \draw[connection] (i) -- (m);
    \draw[connection] (j) -- (l); 
    
    \draw[connection, dashed] (j) -- (x);
    \draw[connection, dashed] (l) -- (y);

    \begin{pgfonlayer}{background}
        \node[highlight, fit=(i) (j) (l) (m), inner sep=0.3cm] (neighborhood) {};
    \end{pgfonlayer}

    \node[red!90!black, font=\bfseries\small] at (neighborhood.north west) [xshift=0.5cm, yshift=0.2cm] {$\cN(i)$};

\end{tikzpicture}}
    \caption{Illustration of Graph-Structured Loss. The shaded region $\cN(i)$
      represents the local neighborhood of criterion $i$. The loss $\ell_{t,i}$
      depends strictly on the state parameters of nodes within this region
      ($\bs|_{\cN(i)}$). Changes to distant, non-neighboring nodes (such as $x$
      and $y$) do not affect the gradient for criterion $i$, enabling
      \textsc{LazyGraphLinUCB} to decompose the global optimization problem into
      local sub-problems.}
    \label{fig:graph_structure}
  \end{wrapfigure}
}
\begin{figure}[t]
  \centering
  \scalebox{1.5}{\begin{tikzpicture}[
    node distance=1.0cm,
    scale=0.5, transform shape,
    criterion/.style={circle, draw=black!80, thick, minimum size=0.9cm, fill=white, font=\bfseries},
    connection/.style={thick, draw=gray!60},
    highlight/.style={draw=red!90, fill=green!5, dashed, rounded corners=15pt, thick}
]

    \node[criterion, fill=blue!15] (i) at (0,0) {$i$};
    
    \node[criterion] (j) [above left=of i] {$j$};
    \node[criterion] (l) [above right=of i] {$l$};
    \node[criterion] (m) [below=of i] {$m$};
    
    \node[criterion, dashed, draw=gray, text=gray] (x) [left=of j] {$x$};
    \node[criterion, dashed, draw=gray, text=gray] (y) [right=of l] {$y$};

    \draw[connection] (i) -- (j);
    \draw[connection] (i) -- (l);
    \draw[connection] (i) -- (m);
    \draw[connection] (j) -- (l); 
    
    \draw[connection, dashed] (j) -- (x);
    \draw[connection, dashed] (l) -- (y);

    \begin{pgfonlayer}{background}
        \node[highlight, fit=(i) (j) (l) (m), inner sep=0.2cm] (neighborhood) {};
    \end{pgfonlayer}

    \node[red!90!black, font=\bfseries\small] at (neighborhood.south east) [xshift=0.3cm, yshift=0.2cm] {$\cN(i)$};

\end{tikzpicture}}
  \caption{Illustration of Graph-Structured Loss. The shaded region $\cN(i)$
    represents the local neighborhood of criterion $i$. The loss $\ell_{t,i}$
    depends strictly on the state parameters of nodes within this region
    ($\bs|_{\cN(i)}$). Changes to distant, non-neighboring nodes (such as $x$
    and $y$) do not affect the gradient for criterion $i$, enabling
    \textsc{LazyGraphLinUCB} to decompose the global optimization problem into
    local sub-problems.}
  \label{fig:graph_structure}
\end{figure}

\section{Algorithm: Lazy Graph-LinUCB}
\label{sec:algorithm}

We propose an algorithm, \textsc{LazyGraphLinUCB}, that efficiently handles the
trade-off between exploration and movement costs by using a lazy update
schedule. The algorithm maintains separate linear estimators for each criterion
but coordinates their updates to minimize global movement.

\subsection{Algorithm Description}

The algorithm, \textsc{LazyGraphLinUCB}
(Algorithm~\ref{alg:lazy_graph_linucb}), extends the standard LinUCB framework
to the loss minimization setting. While the original algorithm was named for the
\emph{Upper Confidence Bound} on rewards, we retain the name to denote the family
of algorithms based on optimism in the face of uncertainty. In our minimization
context, optimism corresponds to using a \emph{Lower Confidence Bound
  (LCB)}. The algorithm operates as follows:

1. Local Estimation: For each criterion $i \in [k]$, maintain a regularized
least-squares estimator for $\btheta^*_i$ using only the features
$\phi_i(\bs)$. Let $\bV_{t,i}$ be the covariance matrix and $\h \btheta_{t,i}$
the estimate at time $t$.

2. Global Optimistic Bound: Construct a global score for the loss function
$L(\bs)$. To ensure exploration, we use the Lower Confidence Bound (LCB),
defined as the estimated loss minus the exploration bonus:
\begin{equation}
  \text{LCB}_t(\bs) = \sum_{i=1}^k \left( \langle \phi_i(\bs), \h \btheta_{t,i} \rangle
    - \beta_{t,i} \| \phi_i(\bs) \|_{\bV_{t,i}^{-1}} \right).
\end{equation}
    
3. Lazy Updates: The algorithm maintains an active state $\bs_{\mathrm{active}}$. It only
re-solves the optimization problem when the determinant of the covariance matrix
for \emph{any} criterion has doubled significantly since the last update. If
$\max_{i} \det(\bV_{t,i}) > 2 \det(\bV_{\mathrm{last}, i})$, then
\begin{equation}
  \label{eq:10}
  \bs_{\mathrm{active}} \leftarrow \argmin_{\bs \in \sS} \text{LCB}_t(\bs).
\end{equation}
Note that the optimization in \eqref{eq:10} is non-convex in general. However,
in many practical settings (e.g., linear or quadratic bases over convex
domains), this optimization is convex or admits efficient approximate solvers.

\begin{algorithm}[t]
   \caption{Synchronous Lazy Graph-LinUCB}
   \label{alg:lazy_graph_linucb}
\begin{algorithmic}[1]
   \STATE \textbf{Input:} Dependency Graph $\cG$, Regularization $\lambda_{\mathrm{reg}}$, Confidence $\delta$.
   \STATE \textbf{Initialize:} For all $i \in [k]$: $\bV_{0,i} = \lambda_{\mathrm{reg}} \bI_{d_i}$, $\h \btheta_{0,i} = \mathbf{0}$, $\tau_{i} = 0$.
   \STATE \textbf{Initialize:} Active state $\bs_{\mathrm{active}} \in \sS$ (arbitrary).
   \FOR{$t=1$ to $T$}
       \STATE \textbf{Check Trigger:}
       \IF{$\exists i \in [k]$ such that $\det(\bV_{t-1, i}) > 2 \det(\bV_{\tau_i, i})$}
           \STATE \emph{// Sufficient information gained; update active policy}
           \STATE Calculate LCB score: $\text{LCB}_t(\bs) = \sum_{j=1}^k (\tri{\phi_j(\bs), \h \btheta_{t-1,j}} - \beta_{t-1,j} \|\phi_j(\bs)\|_{\bV_{t-1,j}^{-1}})$
           \STATE Update Policy: $\bs_{\mathrm{active}} \leftarrow \argmin_{\bs \in \sS} \text{LCB}_t(\bs)$
           \STATE Update Last-Sync Times: $\tau_j \leftarrow t-1$ for all $j \in [k]$
       \ENDIF
       \STATE \textbf{Play:} $\bs_t = \bs_{\mathrm{active}}$
       \STATE \textbf{Observe:} Loss vector $\ell_t(\bs_t)$
     
       \STATE \textbf{Update Estimates:}
       \FOR{$i=1$ to $k$}
           \STATE $\bV_{t, i} = \bV_{t-1, i} + \phi_i(\bs_t)\phi_i(\bs_t)^\top$
           \STATE $\h \btheta_{t, i} = \bV_{t, i}^{-1} \sum_{\tau=1}^t \phi_i(\bs_\tau) \ell_{\tau, i}$
       \ENDFOR
   \ENDFOR
\end{algorithmic}
\end{algorithm}

\textbf{Standard LinUCB and Movement Costs.}  One might ask why the standard
LinUCB algorithm cannot be applied directly to this setting. Recall that at
every round $t$, standard LinUCB selects the state
$\bs_t = \argmin_{\bs \in \sS} \text{LCB}_t(\bs)$.  Since the confidence
intervals and estimated parameters $\h \btheta_t$ evolve stochastically with
every new observation, the global minimizer of the LCB score can oscillate
rapidly. Consequently, standard LinUCB can incur a movement cost of $\Omega(T)$,
yielding linear total regret. To address this, our \textsc{LazyGraphLinUCB}
algorithm freezes the active policy, updating it only when the information gain
(measured by the determinant of the covariance matrix) is sufficient to justify
the cost of switching.

\subsection{Theoretical Analysis}

We now derive the regret bound for \textsc{LazyGraphLinUCB}. The analysis relies
on three key lemmas: establishing confidence ellipsoids, bounding the number of
policy updates (epochs), and bounding the prediction error incurred during the
\emph{lazy} frozen periods.

\subsubsection{Key Lemmas}

\textbf{Confidence Sets.}  We define the confidence sets using the standard
self-normalized martingale bounds. For a criterion $i$, define the regularized
least-squares estimator
$\h \btheta_{t,i} = \bV_{t,i}^{-1} \sum_{\tau=1}^t \phi_i(\bs_\tau) y_{\tau,i}$.

To analyze the regret, we rely on standard self-normalized martingale bounds
\citep{abbasi2011improved} to define confidence ellipsoids $\cC_t$ centered at
$\h \btheta_t$ (see Lemma~\ref{lemma:confidence-ellipsoid} in
Appendix~\ref{app:detailed_proofs}).

Let $\cE$ denote the \emph{good event} that the true parameter $\btheta^*_i$
lies within the confidence ellipsoid $\cC_{t,i}$ for all criteria $i \in [k]$
and all rounds $t \ge 1$.  By a union bound over all $i$ and $t$,
$\Pr(\cE) \ge 1 - \delta$. We condition the rest of the analysis on $\cE$.

\textbf{Movement Bounds.}  The algorithm updates $\bs_{\mathrm{active}}$ only when the
determinant of the covariance matrix doubles. Let $\tau_0, \dots, \tau_M$ be the
update times. We call $[\tau_j, \tau_{j+1}-1]$ the $j$-th \emph{epoch}.

\begin{restatable}[Bound on Number of Epochs]{lemma}{NumberEpochs}
  \label{lemma:number-epochs}
  The total number of updates $M$ is bounded by:
  \[
    M \le \sum_{i=1}^k d_i \log_2 \left( 1 + \frac{T L^2}{d_i
        \lambda_{\mathrm{reg}}} \right).
  \]
\end{restatable}

\begin{proof}
  For a fixed criterion $i$, let $M_i$ be the number of updates triggered by
  this criterion. Each update occurs only when
  $\det(\bV_{t,i}) > 2 \det(\bV_{\mathrm{last},i})$.  The maximum possible determinant is
  bounded by the trace-determinant inequality:
  $\det(\bV_{T,i}) \le (\lambda_{\mathrm{reg}} + T L^2/d_i)^{d_i}$.  Since each
  update doubles the determinant, $2^{M_i} \le
  \det(\bV_{T,i})/\det(\bV_{0,i})$. Taking logs yields
  $M_i \le d_i \log_2 (1 + \frac{T L^2}{d_i \lambda_{\mathrm{reg}}})$. Summing
  over $k$ criteria gives the result.
\end{proof}

\textbf{Prediction Error.}  A key challenge in lazy algorithms is that we play
an action $\bs_t$ based on a \emph{stale} covariance matrix $\bV_{\tau(t)}$. We
must bound the error this staleness introduces.

\begin{restatable}[Bounded Error per Epoch]{lemma}{EpochBound}
  \label{lem:epoch_bound}
  Under the feature normalization assumption $\|\phi\|_{\bV^{-1}} \le 1$, for
  any epoch $j$, the sum of squared stale norms is bounded by:
  \[
    S_j = \sum_{t=\tau_j}^{\tau_{j+1}-1} \| \phi(\bs_t) \|_{\bV_{\tau_j}^{-1}}^2
    \le 3.
  \]
\end{restatable}

\begin{proof}
  Let $\bV_{\mathrm{start}} = \bV_{\tau_j}$. During an epoch, the active policy is
  constant, so $\phi(\bs_t) = \phi_c$ is constant. The accumulated sum of
  squares is $S_j = T_j \|\phi_c\|_{\bV_{\mathrm{start}}^{-1}}^2$, where $T_j$ is the
  epoch length.  By the Matrix Determinant Lemma, the determinant at the end of
  the epoch is $\det(\bV_{\mathrm{end}}) = \det(\bV_{\mathrm{start}})(1 + S_j)$.  The epoch ends
  when the determinant doubles. Even accounting for the discrete step overshoot
  on the final update, the ratio $\det(\bV_{\mathrm{end}})/\det(\bV_{\mathrm{start}})$ is bounded
  by 4. Thus $1+S_j \le 4$, implying $S_j \le 3$.
\end{proof}


\subsubsection{Main Regret Guarantee}

We can now combine these ingredients to bound the total regret.

\begin{restatable}[Regret of Lazy Graph-LinUCB]{theorem}{RegretLazyGraphLinUCB}
  \label{th:regret-lazy-graph-LinUCB}
  Assume the movement cost is bounded by $\lambda$ per switch. Let $d_i$ be the
  dimension of the feature map for criterion $i$, and $d_{\max} = \max_i
  d_i$. The cumulative regret of \textsc{LazyGraphLinUCB} is bounded by:
  \begin{equation}
    \Reg_T
    \le \wt O\paren*{ \sqrt{T} \sum_{i=1}^k d_i + \lambda k^2 d_{\max} \log T }.
  \end{equation}
\end{restatable}

\begin{proof}
  The total regret decomposes into Prediction Regret and Movement Regret:
  \begin{equation}
    \label{eq:inst_regret}
    \Reg_T = \underbrace{\sum_{t=1}^T (L(\bs_t) - L(\bs^*))}_{R_{\mathrm{pred}}}
    + \underbrace{\sum_{t=1}^T \lambda \| \bs_t - \bs_{t-1} \|_1}_{R_{\mathrm{move}}}.
  \end{equation}

  \textbf{1. Movement Regret.} By Lemma~\ref{lemma:number-epochs}, the total
  number of switches $M$ is $\wt O(d_{\max} k \log T)$. Each switch incurs at
  most a cost of $\lambda k$ (since the domain is $[0,1]^k$). Thus,
  $R_{\mathrm{move}} \le \lambda k M = \wt O(\lambda k^2 d_{\max} \log T)$.

  \textbf{2. Prediction Regret.} Using standard optimistic arguments, the
  instantaneous regret is bounded by the confidence width evaluated at the
  \emph{stale} matrix:
  $r_t \le 2 \sum_{i=1}^k \beta_{T,i} \|\phi_i(\bs_t)\|_{\bV_{\tau(t),i}^{-1}}$.
  Summing over all epochs $j=1, \dots, M_i$ for criterion $i$:
  \begin{align*}
    \sum_{t=1}^T \| \phi_i(\bs_t) \|_{\bV_{\tau(t),i}^{-1}} 
    & \le \sum_{j=1}^{M_i} \sqrt{T_j \cdot S_j} \quad (\text{by Cauchy-Schwarz on epoch } j) \\
    & \le \sqrt{3} \sum_{j=1}^{M_i} \sqrt{T_j} \quad (\text{by Lemma~\ref{lem:epoch_bound}}).
  \end{align*}
  This sum is maximized when $T_j = T/M_i$, yielding a bound of
  $\sqrt{3 M_i T}$. Since $M_i = \wt O(d_i)$, and
  $\beta_{T,i} = \wt O(\sqrt{d_i})$, the total prediction regret scales as:
  \[
    R_{\mathrm{pred}} \le \sum_{i=1}^k \beta_{T,i} \sqrt{3 M_i T} = \wt O\left(
      \sqrt{T} \sum_{i=1}^k d_i \right).
  \]
  Adding the movement and prediction terms yields the final bound.
\end{proof}

\begin{restatable}[Robustness to Graph
  Misspecification]{corollary}{Misspecification}
  \label{cor:misspecification}
  Suppose the learner operates using a proxy graph $\cG'$ such that
  $E \subseteq E'$ (a super-graph). Let $d'_i$ denote the feature dimension
  induced by $\cG'$. Then, \textsc{LazyGraphLinUCB} remains valid and guarantees
  a regret bound scaling with the proxy dimensions:
  \begin{equation}
    \Reg_T
    \le \wt O\left( \sqrt{T} \sum_{i=1}^k d'_i + \lambda k^2 d'_{\max} \log T \right).
  \end{equation}
\end{restatable}

\begin{proof}
  The proof relies on the fact that any linear model defined on a graph $\cG$ is
  also realizable on any super-graph $\cG'$. By
  Definition~\ref{def:graph_struct}, the true expected loss is
  $\mu_i(\bs) = \tri{ \phi_i(\bs|_{\cN(i)}), \btheta^*_i }$ for some
  $\btheta^*_i \in \Rset^{d_i}$. Since $E \subseteq E'$, we have
  $\cN(i) \subseteq \cN'(i)$. The features constructed for the proxy graph,
  $\phi'_i(\bs|_{\cN'(i)})$, span a space $\Rset^{d'_i}$ that contains the
  subspace spanned by the true features $\phi_i$.  Formally, there exists a
  linear embedding such that the true parameter $\btheta^*_i$ can be represented
  as a vector $\wt \btheta_i \in \Rset^{d'_i}$ (essentially padding the extra
  dimensions with zeros), satisfying
  $\tri{ \phi'_i(\bs), \wt \btheta_i } = \mu_i(\bs)$ for all
  $\bs$. Consequently, the problem remains a valid linear bandit instance with
  dimension $d'_i$.

  The regret bound of Theorem~\ref{th:regret-lazy-graph-LinUCB} depends on the
  problem dimension through two terms: the confidence radius $\beta_{t,i}$
  (Lemma~\ref{lemma:confidence-ellipsoid}) and the number of updates $M$
  (Lemma~\ref{lemma:number-epochs}).

  \begin{enumerate}

  \item Validity of Confidence Sets: Since the problem is realizable in
    $\Rset^{d'_i}$, Lemma~\ref{lemma:confidence-ellipsoid} applies to the proxy
    estimator $\h \btheta'_{t,i}$, ensuring that with high probability,
    $\wt \btheta_i$ lies within the confidence ellipsoid defined by the proxy
    covariance matrix $\bV'_{t,i}$. The radius $\beta'_{t,i}$ now scales with
    $\sqrt{d'_i}$ instead of $\sqrt{d_i}$.

  \item Update Bound: Lemma~\ref{lemma:number-epochs} bounds the number of
    determinant doublings for a matrix of dimension $d$. Applying this to the
    $d'_i$-dimensional proxy covariance matrices yields a total update bound
    $M \le \sum_{i=1}^k d'_i \log \left( 1 + \frac{T L^2}{d'_i
        \lambda_{\mathrm{reg}}} \right)$.
  \end{enumerate}
  Substituting the proxy dimension $d'_i$ into the regret decomposition
  \eqref{eq:inst_regret} yields the stated bound. This confirms that the
  algorithm is safe: over-estimating the graph structure merely increases the
  regret bound by a factor depending on the dimension inflation (scaling with
  $d'_i/d_i$ in the prediction term and $d'_i/d_i$ in the movement term),
  without introducing bias.
\end{proof}

\subsection{Minimax Lower Bound}

We now present a lower bound demonstrating that the dependency on the sum of
dimensions in our upper bound is tight.
\begin{restatable}[Minimax Lower Bound]{theorem}{LowerBound}
  \label{th:lower-bound}
  For any learning algorithm, there exists a graph structure $\cG$ and a
  sequence of loss functions such that the expected regret is lower bounded by:
  \begin{equation}
    \E[\Reg_T] = \Omega\left( \sqrt{T} \sum_{i=1}^k d_i \right).
  \end{equation}
\end{restatable}

\begin{proof}
  Consider a graph $\cG$ consisting of $k$ disjoint components (cliques), where
  each component $i$ corresponds to an independent linear bandit problem of
  dimension $d_i$. The loss function decomposes as
  $L(\bs) = \sum_{i=1}^k L_i(\bs|_{\cN(i)})$, where each $L_i$ is governed by an
  independent parameter $\btheta^*_i$. It is a standard result in the bandit
  literature (e.g., \citep{dani2008stochastic}) that the minimax lower bound for
  a single $d$-dimensional linear bandit is $\Omega(d \sqrt{T})$. Since the
  problems are independent, the learner must solve each simultaneously. The
  total regret is thus lower bounded by the sum of the regrets of the individual
  components:
  \[
    \E[\Reg_T] \ge \sum_{i=1}^k \Omega(d_i \sqrt{T}) = \Omega\left( \sqrt{T}
      \sum_{i=1}^k d_i \right)
  \]
  This confirms that the dependency on the sum of dimensions in
  our upper bound is tight.
\end{proof}

\textbf{Comparison to Naive Reduction.}
One might seek to solve this problem by scalarizing the objective
$L(\bs) = \sum \ell_i(\bs)$ and applying a standard high-dimensional linear
bandit algorithm to the concatenated parameter $\btheta^* \in \Rset^{\sum
  d_i}$. While the dimension dependence ($\sum d_i$) would match our result,
this naive reduction fails on two fronts. First, \emph{Feedback Granularity}: By
observing only the scalar sum, the learner faces an aggregated noise variance of
$\sigma_{\mathrm{agg}}^2 = \sum_{i=1}^k \sigma_i^2 \approx k$. This inflates the regret
by a factor of $\sqrt{k}$ compared to our component-wise approach which exploits
the vector-valued feedback. Second, and most importantly, \emph{Stability}: The
global approach cannot distinguish which local component caused a loss
increase. Consequently, it updates the entire parameter vector at once,
incurring a movement cost proportional to the global dimension $k$ at every
switch, whereas our approach restricts movement costs to local neighborhoods.

\section{Exploiting Graph Structure}
\label{sec:exploit_graph_structure}

While the global Lazy strategy (Section~\ref{sec:algorithm}) achieves optimal
prediction regret, it does not fully leverage the sparsity and correlations
inherent in the dependency graph $\cG$. In this section, we present three key
extensions that exploit this structure: an \emph{asynchronous} update schedule
that minimizes movement costs in sparse graphs, an \emph{adaptive} procedure
that learns the graph structure when it is initially unknown, and a
\emph{factor-graph} decomposition that leverages data sharing among correlated
objectives to tighten regret bounds.
Table~\ref{tab:variant_guide} summarizes when each variant is best suited and
what additional assumptions it requires.

All detailed proofs are deferred to Appendix~\ref{app:detailed_proofs}.

\begin{table}[t]
  \centering
  \caption{\textbf{Algorithm Variant Selection Guide.} Summary of when to use
    each variant and what additional assumptions it requires.}
  \label{tab:variant_guide}
  \vspace{0.1cm}
  \begin{small}
      \begin{tabular}{@{}p{2.62cm}p{3.2cm}p{3.2cm}@{}}
        \toprule
        \textbf{Variant} & \textbf{Best When} & \textbf{Key Assumption} \\
        \midrule
        Global Lazy\newline(Section~\ref{sec:algorithm}) &
                                                     Dense graph or synchronized updates needed &
                                                                                                  None beyond standard linear bandit \\
        Async Lazy\newline(Section~\ref{sec:async_extension}) &
                                                          Sparse/heterogeneous graph; criteria learn at different rates &
                                                                                                                          Assumption A: LCB objective is PL (satisfied for large $\lambda_{\mathrm{reg}}$) \\
        Adaptive Refinement\newline(Section~\ref{sec:graph_learning}) &
                                                                  Graph structure completely unknown &
                                                                                                       Assumptions B (signal strength) and C (diverse contexts) \\
        Joint Estimator\newline(Section~\ref{sec:data_sharing}) &
                                                            Dense graph with correlated objectives sharing latent factors &
                                                                                                                            Assumption D: factor decomposability \\
        \bottomrule
      \end{tabular}
  \end{small}
  \vskip -.1in
\end{table}

\subsection{Asynchronous Updates for Sparse Graphs}
\label{sec:async_extension}

In the \textsc{LazyGraphLinUCB} algorithm described above, the update condition
is \emph{global}: if any single criterion triggers a determinant doubling, the
entire policy vector $\bs_{\mathrm{active}}$ is re-optimized. While this guarantees
regret optimality in the worst case (dense graphs), it may be inefficient for
systems with sparse dependency graphs. For instance, learning new information
about a latency criterion should ideally not force a reconfiguration of a
graph-distant fairness criterion. To exploit sparsity, we propose an
asynchronous algorithm, \textsc{Async-LazyGraphLinUCB}, which performs
\emph{local} policy updates. Such a solution is necessary, as global lazy
updates are suboptimal for sparse graphs.

\paragraph{Mechanism.} The algorithm maintains the same local estimators but
modifies the update schedule:

1. Local Trigger: Each criterion $i$ maintains its own last-update time
$\tau_i$. An update is triggered for criterion $i$ only if
$\det(\bV_{t,i}) > 2 \det(\bV_{\tau_i, i})$.

2. Local Update: When criterion $i$ triggers, we solve for the new active policy
$\bs_{\mathrm{active}}$ by optimizing only the variables in the local neighborhood
$\cN(i)$, keeping all variables $s_j$ for $j \notin \cN(i)$ fixed at their
previous values.

\emph{Tie-Breaking:} In the event that multiple criteria trigger simultaneously
(i.e., $|\cA_t| > 1$), we perform a synchronized update over the union of their
neighborhoods $\bigcup_{i \in \cA_t} \cN(i)$. This ensures that coupled
dependencies between simultaneously triggering nodes are resolved jointly.

3. Synchronization: The time $\tau_i$ is updated to $t$.  Note that because
neighborhoods overlap, an update to $\cN(i)$ may effectively update partial
state for neighbors $j \in \cN(i)$, but it does not propagate to the entire
graph.

\textbf{Assumption A: BCD Convergence.} We assume the objective function is
sufficiently regular, specifically, that it satisfies the Polyak-{\L}ojasiewicz
(PL) condition, such that the Block Coordinate Descent (BCD) updates in the
asynchronous procedure converge to the global minimizer with negligible
optimality gap.

\emph{When does Assumption A hold?}  The LCB objective
$f_t(\bs) = \sum_i (\langle \phi_i, \h\btheta_{t,i}\rangle -
\beta_{t,i}\|\phi_i(\bs)\|_{\bV_{t,i}^{-1}})$ is the sum of a \emph{strongly
  convex} empirical loss term (the regularizer contributes
$\lambda_{\mathrm{reg}}\|\bs\|^2/2$ per coordinate) and a concave correction
(the exploration bonus). For any smooth, bounded feature map $\phi_i$ on a
compact domain, the Hessian of the exploration bonus is bounded in operator norm
by some constant $C(\phi_i)$ that depends only on the Lipschitz constant of
$\phi_i$ and the problem geometry.  Consequently, whenever
$\lambda_{\mathrm{reg}} > C(\phi_i)$, the net Hessian is positive definite and
the LCB objective is $(\lambda_{\mathrm{reg}} - C(\phi_i))$-strongly convex,
hence PL (see Lemma~\ref{lemma:assumption-a-justified} in
Appendix~\ref{app:detailed_proofs}). In particular, as the exploration bonus
shrinks over time ($\beta_{t,i} \to 0$), this regime is automatically
entered. The algorithm is designed to operate in this regularity regime by an
appropriate choice of $\lambda_{\mathrm{reg}}$.

\textbf{Theoretical Improvement.}
This asynchronous strategy yields a significantly tighter bound on the movement
cost for sparse graphs. In the global strategy, every update incurs a movement
cost bounded by the total dimension $k$ (since all $s_i$ might shift). In the
asynchronous strategy, an update initiated by criterion $i$ only shifts
coordinates within $\cN(i)$, bounding the movement cost by the local degree
$|\cN(i)|$.

\begin{restatable}[Regret of Asynchronous Lazy
  Graph-LinUCB]{theorem}{AsyncRegret}
  \label{th:async_regret}
  Let $\Delta = \max_i |\cN(i)|$ be the maximum degree of the dependency
  graph. In the regularity regime where $\lambda_{\mathrm{reg}}$ is chosen
  sufficiently large to satisfy Assumption~A (see
  Lemma~\ref{lemma:assumption-a-justified}), the cumulative regret of the
  asynchronous algorithm is bounded by:
  \begin{equation}
    \Reg_T
    \le \wt O\paren*{ \sqrt{T} \sum_{i=1}^k d_i
      + \lambda \Delta d_{\max} k \log T }.
  \end{equation}
\end{restatable}

\begin{proof}
  We analyze the prediction regret and movement regret separately.

  \textbf{Part~(a): Prediction Regret (Feature Stability).} The asynchronous algorithm maintains the invariant
  that for any criterion $i$, the active policy was computed using a covariance
  matrix $\bV_{\tau_i, i}$ such that
  $\det(\bV_{t,i}) \le 2 \det(\bV_{\tau_i, i})$.

  Unlike the global setting, the feature vector $\phi_i(\bs_t)$ may vary during
  an epoch because updates to neighbors $j \in \cN(i)$ can alter the state
  $\bs|_{\cN(i)}$. However, the bound on the sum of squared norms derived in
  Lemma~\ref{lem:epoch_bound} applies even under time-varying features.

  Let $\bG_j$ be the matrix of diverse feature vectors observed for criterion
  $i$ during its $j$-th epoch. The sum of squared stale norms is exactly the
  trace $\text{Tr}(\bG_j^\top \bV_{\tau_j, i}^{-1} \bG_j)$. Let $\mu_k$ be the
  eigenvalues of the matrix
  $\bV_{\tau_j, i}^{-1/2} \bG_j \bG_j^\top \bV_{\tau_j, i}^{-1/2}$. The
  determinant doubling condition implies:
  \begin{equation}
    \det(\bI + \bG_j^\top \bV_{\tau_j, i}^{-1} \bG_j)
    = \prod_k (1 + \mu_k)
    \le \frac{\det(\bV_{\tau_{j+1}, i})}{\det(\bV_{\tau_j, i})}
    \le 2.
  \end{equation}
  Since $\mu_k \ge 0$ and $\prod (1+\mu_k) \le 4$ (allowing for discrete
  overshoot), we must have $\sum \mu_k \le 3$.  Using the inequality
  $x \le 3 \log(1+x)$ which holds for $x \in [0, 3]$ we bound the trace:
  \begin{equation}
    \sum_{t \in \text{epoch } j} \|\phi_i(\bs_t)\|_{\bV_{\tau_j, i}^{-1}}^2
    = \sum_k \mu_k \le 3 \sum_k \log(1+\mu_k)
    = 3 \log \det(\bI + \bG_j^\top \bV_{\tau_j, i}^{-1} \bG_j) \le 3 \log 4.
  \end{equation}
  Thus, the aggregate error contribution of any epoch remains bounded by
  $3 \log 4$ regardless of feature variance within the epoch. The rest of the
  summation proceeds identically to the global case, yielding a prediction
  regret of $\wt O(\sqrt{T}\sum_i d_i)$.

  \textbf{Part~(b): Prediction Regret (Optimism).} The asynchronous algorithm maintains the invariant
  that the active policy $\bs_t$ is the minimizer of the LCB function. Under
  Assumption A, the block-coordinate updates ensure that $\bs_t$ converges to
  the global optimum, preserving the validity of the standard optimistic regret
  bound.  Specifically, at any time $t$, for any criterion $i$, the ``lazy''
  condition $\det(\bV_{t,i}) \le 2 \det(\bV_{\tau_i, i})$ ensures that the
  covariance matrix used to compute the active policy ($\bV_{\tau_i, i}$) is
  spectrally similar to the current true covariance ($\bV_{t,i}$). Crucially,
  since the objective $\mu_i(\bs)$ and its LCB depend strictly on the local
  variables $\bs|_{\cN(i)}$, optimizing over the neighborhood $\cN(i)$ while
  keeping disjoint variables fixed is sufficient to minimize the local LCB
  contribution.  The analysis in Lemma~\ref{lem:epoch_bound} applies locally to
  each criterion $i$. The sequence of updates for criterion $i$ defines a set of
  local epochs. For any $t$, let $\tau_i(t)$ be the last time criterion $i$
  triggered an update. The instantaneous regret bound \eqref{eq:inst_regret}
  becomes:
  \begin{equation}
    r_t \le 2 \sum_{i=1}^k \beta_{T,i} \|\phi_i(\bs_t)\|_{\bV_{\tau_i(t), i}^{-1}}.
  \end{equation}
  Since the trigger condition holds, the bound on the sum of squared norms
  holds. Multiplying by the confidence radius $\beta_{T,i} \approx \sqrt{d_i}$,
  the prediction regret sums to $\tilde{O}(\sqrt{T} \sum_i d_i)$, identical to
  the global case.

  \textbf{Part~(c): Movement Regret.}  Unlike the global case, an update does not necessarily
  change the entire state vector $\bs$. Let $u_{t,i} \in \{0, 1\}$ be an
  indicator that criterion $i$ triggered an update at time $t$. The total number
  of triggers for criterion $i$ is bounded by Lemma~\ref{lemma:number-epochs} as
  $M_i = \sum_{t=1}^T u_{t,i} \le \wt O(d_i \log T)$. When criterion $i$
  triggers, the algorithm updates $\bs_{\mathrm{active}}$ by optimizing over the local
  neighborhood $\cN(i)$. Let $\cA_t = \{i \colon u_{t,i} = 1\}$ be the set of
  triggering criteria at time $t$. The coordinates of $\bs$ that change are
  restricted to the union of neighborhoods:
  \begin{equation}
    \supp(\bs_t - \bs_{t-1}) \subseteq \bigcup_{i \in \cA_t} \cN(i).
  \end{equation}
  The movement cost at step $t$ is bounded by the number of changed coordinates:
  \begin{equation}
    \norm*{\bs_t - \bs_{t-1}}_1
    \le \left| \bigcup_{i \in \cA_t} \cN(i) \right|
    \le \sum_{i \in \cA_t} |\cN(i)|
    \le \sum_{i \in \cA_t} \Delta.
  \end{equation}
  Summing over $T$:
  \begin{equation}
    R_{\mathrm{move}}
    = \lambda \sum_{t=1}^T \|\bs_t - \bs_{t-1}\|_1
    \le \lambda \sum_{t=1}^T \sum_{i \in \cA_t} \Delta
    = \lambda \Delta \sum_{i=1}^k \sum_{t=1}^T u_{t,i}
    = \lambda \Delta \sum_{i=1}^k M_i.
  \end{equation}
  Substituting $M_i \le \wt O(d_{\max} \log T)$ yields
  $R_{\mathrm{move}} \le \wt O(\lambda \Delta k d_{\max} \log T)$.
\end{proof}

This result highlights that for sparse graphs (where $\Delta \ll k$), the
asynchronous strategy effectively decouples the learning problems, reducing the
dependency on the global dimension $k$ in the movement term. Note that
determinant-based triggers are a standard tool for computational efficiency in
linear bandits; we repurpose them here for stability. Crucially, a naive
application to graph-structured problems leads to a synchronization penalty,
where stable criteria update unnecessarily. Our primary contribution is the
asynchronous analysis (Theorem~\ref{th:async_regret}), which proves that
decomposing the trigger conditions is essential for optimal movement costs in
heterogeneous systems.

\subsection{Adaptive Graph Learning}
\label{sec:graph_learning}

In many applications, the true dependency graph $\cG$ may be unknown. While
relying on a safe super-graph ensures validity
(Corollary~\ref{cor:misspecification}), it incurs suboptimal movement
costs. Here, we show that when the true dependencies satisfy a minimum signal
strength condition and the contexts are sufficiently diverse, similar to the
\emph{compatibility conditions} required for high-dimensional bandits
\citep{bastani2020online}, the learner can adaptively recover the sparse graph
structure.

The adaptive graph learning mechanism in this section is primarily a theoretical
guarantee showing that a priori knowledge of the dependency graph is not
strictly necessary for optimal regret. Our goal is not to propose a
\emph{turnkey algorithm} for structure learning under movement costs, but to
establish that, under standard identifiability conditions, the learner can
provably recover and exploit sparsity without compromising regret
rates. Practical variants and empirical evaluation are left to future work.

\textbf{Assumption B: Minimum Signal Strength.} Let $\phi_{\mathrm{full}}(\bs)$
be the feature map corresponding to the complete graph. We assume the true
parameter $\btheta^*_{i, \mathrm{full}}$ is supported on a sparse subset of
indices corresponding to the true neighborhood $\cN(i)$. For any active neighbor
$j \in \cN(i)$, the dependency is bounded away from zero:
$|\theta^*_{i,j}| \ge \gamma > 0$.

\textbf{Assumption C: Diverse Contexts.} During the warm-up phase, we assume the
system visits a diverse set of states such that the minimum eigenvalue of the
covariance matrix grows linearly. Formally, there exists $\kappa > 0$ such that
for any $\tau$, $\lambda_{\min}(\bV_{\tau, i}) \ge \kappa \tau$. This can be
enforced by adding isotropic noise to the actions during the warm-up phase.

\textbf{Algorithm: Adaptive Graph Refinement.}  We propose a two-phase strategy:

1.\ Warm-up Phase: Initialize with the complete graph $\cG_{\mathrm{dense}}$ with
dimension $d_{\mathrm{total}}$. Run \textsc{LazyGraphLinUCB} for $\tau$ rounds.

2.\ Sparsity Test: At $t=\tau$, for each criterion $i$, compute the estimator
$\h \btheta_{\tau, i}$ using the dense covariance matrix. Construct the
estimated neighborhood
$\h \cN(i) = \{j : |\h \theta_{\tau, i, j}| > \gamma/2\}$.

3.\ Refinement: Update the graph to $\h \cG = (\cV, \h E)$ based on these
neighborhoods. For $t > \tau$, resume \textsc{LazyGraphLinUCB} using $\h \cG$
(retaining the covariance matrices $\bV_{\tau, i}$).

\begin{restatable}[Regret of Adaptive Graph Refinement]{theorem}{AdaptiveRegret}
  \label{th:adaptive_regret}
  Let $\cG_{\mathrm{dense}}$ be the initial dense graph and $\cG$ be the true sparse
  graph. Set the warm-up length to:
  $\tau = \frac{16 \beta_{T}^2}{\kappa \gamma^2}$, where $\beta_{T}$ is the
  confidence radius for dimension $d_{\mathrm{total}}$ as defined in
  Lemma~\ref{lemma:confidence-ellipsoid}. Under Assumptions B and C, with
  probability at least $1-\delta$, the algorithm recovers the true graph
  $\h \cG = \cG$, and the total regret is bounded by:
  \begin{equation}
    \Reg_T
    \le \wt O\left( \frac{d_{\mathrm{total}}^{1.5}}{\kappa \gamma^2}
      + \sqrt{T} \sum_{i=1}^k d_i
      + d_{\mathrm{total}}^2 \log T \right).
  \end{equation}
\end{restatable}

\begin{proof}
  The proof proceeds in three steps: establishing the estimation error bound,
  proving support recovery, and summing the regret. We start with the guarantee
  provided by the confidence ellipsoids. From
  Lemma~\ref{lemma:confidence-ellipsoid}, with probability $1-\delta$, for all
  $i$:
  \begin{equation}
    \| \h \btheta_{\tau, i} - \btheta^*_i \|_{\bV_{\tau, i}}
    \le \beta_{\tau, i}.
  \end{equation}
  We relate the Mahalanobis norm to the Euclidean norm using the minimum
  eigenvalue of $\bV_{\tau, i}$. By the definition of the matrix norm,
  $\|\bx\|_{\bV}^2 \ge \lambda_{\min}(\bV) \|\bx\|_2^2$. Therefore:
  \begin{equation}
    \| \h \btheta_{\tau, i} - \btheta^*_i \|_2
    \le \frac{1}{\sqrt{\lambda_{\min}(\bV_{\tau, i})}}
    \| \h \btheta_{\tau, i} - \btheta^*_i \|_{\bV_{\tau, i}}
    \le \frac{\beta_{\tau, i}}{\sqrt{\kappa \tau}}.
  \end{equation}
  Since the $\ell_\infty$ norm is bounded by the $\ell_2$ norm
  ($\|\bx\|_\infty \le \|\bx\|_2$), we have the element-wise bound for any
  component $j$:
  \begin{equation}
    |\h \theta_{\tau, i, j} - \theta^*_{i, j}|
    \le \frac{\beta_{\tau, i}}{\sqrt{\kappa \tau}}.
  \end{equation}
  We set $\tau$ such that the error bound is strictly less than
  $\gamma/2$. Substituting $\tau = 16 \beta_T^2 / (\kappa \gamma^2)$ into the
  error bound yields:
  \begin{equation}
    |\h \theta_{\tau, i, j} - \theta^*_{i, j}|
    \le \frac{\beta_{\tau, i}}{\sqrt{\kappa \frac{16 \beta_T^2}{\kappa \gamma^2}}}
    = \frac{\gamma \beta_{\tau, i}}{4 \beta_T}
    \le \frac{\gamma}{4}.
  \end{equation}
  We now consider the thresholding rule:
  \begin{itemize}

  \item Case 1 (Active): If $j \in \cN(i)$, then $|\theta^*_{i, j}| \ge
    \gamma$. The triangle inequality implies
    $|\h \theta_{\tau, i, j}| \ge |\theta^*_{i, j}| - |\text{error}| \ge \gamma
    - \gamma/4 = 3\gamma/4$. Since $3\gamma/4 > \gamma/2$, the index $j$ is
    correctly included.

  \item Case 2 (Inactive): If $j \notin \cN(i)$, then $\theta^*_{i, j} = 0$. The
    error bound implies $|\h \theta_{\tau, i, j}| \le \gamma/4$. Since
    $\gamma/4 < \gamma/2$, the index $j$ is correctly excluded.

\end{itemize}
Thus, at $t=\tau$, we have $\h \cG = \cG$. The total regret decomposes into the
warm-up phase and the refined phase.  During the warm-up ($t=1 \dots \tau$), we
play isotropic noise, incurring per-round regret
$O(\sqrt{d_{\mathrm{total}}})$ (the confidence width of the dense
estimator). The total warm-up regret is thus
$O(\tau \cdot \sqrt{d_{\mathrm{total}}})$. Substituting
$\tau \approx d_{\mathrm{total}}/(\kappa \gamma^2)$, this cost is
$\wt O(d_{\mathrm{total}}^{1.5} / (\kappa \gamma^2))$.  For the remaining $T-\tau$
rounds, we run the algorithm on the true sparse graph, incurring regret
$\wt O(\sum_{i = 1}^k d_i \sqrt{T})$. Summing these yields the final bound.
\end{proof}

\textbf{Regime of Validity.}
The regret bound consists of a constant \emph{learning cost} (scaling with
$\gamma^{-2}$) and a term scaling with $\sqrt{T}$. This bound is meaningful when
the horizon is sufficiently long to amortize the warm-up phase, specifically
when $T \gg \wt \Omega(\gamma^{-4})$. In the regime of extremely weak signals
($\gamma \le T^{-1/4}$), the cost of distinguishing the true support outweighs
the benefits of sparsity within the horizon $T$. In such cases, the learner is
better off defaulting to the conservative super-graph strategy
(Corollary~\ref{cor:misspecification}), effectively treating weak signals as
noise.

\subsection{Data Sharing for Correlated Objectives}
\label{sec:data_sharing}

The regret bounds in Section~\ref{sec:algorithm} scale with
$\sum_{i = 1}^k d_i$, treating each criterion as an independent learning
task. However, the graph structure often implies stronger correlations:
neighbors $i$ and $j$ in the dependency graph typically depend on the same
underlying state variables or latent factors. In this section, we show that by
formalizing the global structure explicitly as a \emph{Factor Graph}
decomposition, a standard realistic assumption in graphical models (see e.g.,
\citep{CortesKuznetsovMohriScott2016}), we can further reduce the regret.

\textbf{Assumption D: Factor Decomposability.}  We assume the global loss
function $L(\bs)$ decomposes over the maximal cliques $\cC$ of the graph $\cG$:
\begin{equation}
  L(\bs)
  = \sum_{C \in \cC} \tri{ \psi_C(\bs|_C), \bw^*_C },
\end{equation}
where $\psi_C$ are local feature maps for each clique and
$\bw^*_C \in \Rset^{d_C}$ are unique clique parameters. This differs from the
node-wise assumption in \eqref{eq:1}
($\mu_i(\bs) = \tri{\phi_i, \btheta^*_i}$).  \ignore{
  \setlength{\intextsep}{-1pt} \setlength{\columnsep}{10pt}
  \begin{wrapfigure}{r}{0.4\columnwidth}
    \centering \resizebox{.4\columnwidth}{!}{    \begin{tikzpicture}[scale=1.2]
        \coordinate (n1) at (-1.5, 0);
        \coordinate (n2) at (0, 1);   
        \coordinate (n3) at (0, -1);  
        \coordinate (n4) at (1.5, 0);

        \fill[red!20] (n1) -- (n2) -- (n3) -- cycle;
        \node[red!80] at (-0.6, 0) {\small $\mathcal{C}_1$};

        \fill[blue!20] (n4) -- (n2) -- (n3) -- cycle;
        \node[blue!80] at (0.6, 0) {\small $\mathcal{C}_2$};

        \draw[thick, gray] (n1) -- (n2);
        \draw[thick, gray] (n1) -- (n3);
        \draw[thick, gray] (n2) -- (n4);
        \draw[thick, gray] (n3) -- (n4);
        \draw[thick, black] (n2) -- (n3); 

        \node[draw, circle, fill=white, inner sep=2pt] at (n1) {\footnotesize 1};
        \node[draw, circle, fill=gray!30, inner sep=2pt] at (n2) {\footnotesize 2};
        \node[draw, circle, fill=gray!30, inner sep=2pt] at (n3) {\footnotesize 3};
        \node[draw, circle, fill=white, inner sep=2pt] at (n4) {\footnotesize 4};

        \node[anchor=north] at (0, -1.2) {\footnotesize Overlap: $\{2, 3\}$};
    \end{tikzpicture}}
    \caption{\textbf{Factor Graph Decomposition.} Illustration of two
      overlapping cliques. Nodes 2 and 3 participate in both interaction $\cC_1$
      and $\cC_2$. While a node-wise estimator treats these dependencies
      independently (counting the shared dimension twice), the Joint Estimator
      exploits the overlap to learn unique parameters for each clique, reducing
      the effective dimension $d_{\mathrm{eff}}$.}
    \label{fig:clique_overlap}
  \end{wrapfigure}
}
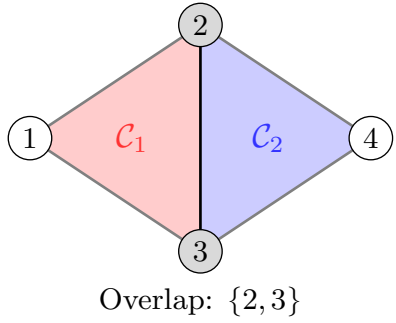
\begin{figure}[t]
  \centering
  \scalebox{1.25}{    \begin{tikzpicture}[scale=1.2]
        \coordinate (n1) at (-1.5, 0);
        \coordinate (n2) at (0, 1);   
        \coordinate (n3) at (0, -1);  
        \coordinate (n4) at (1.5, 0);

        \fill[red!20] (n1) -- (n2) -- (n3) -- cycle;
        \node[red!80] at (-0.6, 0) {\small $\mathcal{C}_1$};

        \fill[blue!20] (n4) -- (n2) -- (n3) -- cycle;
        \node[blue!80] at (0.6, 0) {\small $\mathcal{C}_2$};

        \draw[thick, gray] (n1) -- (n2);
        \draw[thick, gray] (n1) -- (n3);
        \draw[thick, gray] (n2) -- (n4);
        \draw[thick, gray] (n3) -- (n4);
        \draw[thick, black] (n2) -- (n3); 

        \node[draw, circle, fill=white, inner sep=2pt] at (n1) {\footnotesize 1};
        \node[draw, circle, fill=gray!30, inner sep=2pt] at (n2) {\footnotesize 2};
        \node[draw, circle, fill=gray!30, inner sep=2pt] at (n3) {\footnotesize 3};
        \node[draw, circle, fill=white, inner sep=2pt] at (n4) {\footnotesize 4};

        \node[anchor=north] at (0, -1.2) {\footnotesize Overlap: $\{2, 3\}$};
    \end{tikzpicture}}
  \caption{\textbf{Factor Graph Decomposition.} Illustration of two overlapping
    cliques. Nodes 2 and 3 participate in both interaction $\cC_1$ and
    $\cC_2$. While a node-wise estimator treats these dependencies independently
    (counting the shared dimension twice), the Joint Estimator exploits the
    overlap to learn unique parameters for each clique, reducing the effective
    dimension $d_{\mathrm{eff}}$.}
  \vskip -.15in
  \label{fig:clique_overlap}
\end{figure}
There, $\btheta^*_i$ aggregates all factors affecting node $i$. Here, we
disentangle them. For instance, if $i$ and $j$ share a clique $C$, they share
the parameter $\bw^*_C$. This is physically realistic in systems where
constraints share underlying physics (e.g., two fairness constraints depending
on the same sensitive attribute coefficients).

\textbf{Algorithm: Joint-LazyGraphLinUCB.}  We construct a \emph{Joint
  Estimator} by concatenating all unique clique parameters into a single vector
$\btheta_{joint}^* \in \Rset^{d_{\mathrm{eff}}}$, where
$d_{\mathrm{eff}} = \sum_{C \in \cC} d_C$. At each round $t$, the learner
receives a feedback vector $\bY_t = \ell_t(\bs_t) \in \Rset^k$. We treat this as
a single linear bandit problem with a \emph{macro-action} $\bs_t$ and a
vector-valued reward. The estimator minimizes the joint regularized least
squares objective. The update rule is global: $\bs_{\mathrm{active}}$ is updated only
when the determinant of the \emph{joint} covariance matrix doubles.  

\begin{restatable}[Regret of Joint Estimator]{theorem}{JointRegret}
  \label{th:joint_regret}
  Let $d_{\mathrm{eff}} = \sum_{C \in \cC} d_C$ be the sum of the dimensions of
  the maximal cliques. The cumulative regret of the Joint-LazyGraphLinUCB is
  bounded by:
  \begin{equation}
    \Reg_T
    \le \wt O\left( d_{\mathrm{eff}} \sqrt{k T}
      + \lambda k^2 d_{\max} \log T \right).
  \end{equation}
\end{restatable}

\begin{proof}
  We analyze the prediction regret by treating the problem as a linear bandit
  with matrix-valued features. At each round $t$, the learner observes
  $\bY_t = \Psi(\bs_t) \btheta_{joint}^* + \boldsymbol{\eta}_t \in \Rset^k$,
  where $\Psi(\bs_t) \in \Rset^{k \times d_{\mathrm{eff}}}$ is the block-sparse
  feature matrix where the $c$-th block contains the features
  $\psi_C(s|_C)$. The ridge estimator $\h \btheta_t$ minimizes the joint
  regularized objective, with covariance
  $\bV_t = \lambda I + \sum_{\tau=1}^t \Psi(\bs_\tau)^\top
  \Psi(\bs_\tau)$. Standard self-normalized martingale bounds for vector-valued
  martingales \citep{abbasi2011improved} guarantee that with high probability,
  $\|\h \btheta_t - \btheta_{joint}^*\|_{\bV_t} \le \beta_t$, where
  $\beta_t = \wt O(\sqrt{d_{\mathrm{eff}}})$.
  
  The instantaneous regret is bounded by the confidence width of the joint
  estimator using standard OFUL arguments:
  \begin{equation}
    r_t = L(\bs_t) - L(\bs^*) \le 2 \beta_t \|\Psi(\bs_t)^\top \mathbf{1}\|_{\bV_t^{-1}} = 2 \beta_t \sqrt{\mathbf{1}^\top (\Psi(\bs_t) \bV_t^{-1} \Psi(\bs_t)^\top) \mathbf{1}}.
  \end{equation}
  Crucially, for any PSD matrix $\bA$, we have the algebraic bound
  $\mathbf{1}^\top \bA \mathbf{1} \le k \cdot \text{tr}(\bA)$. This introduces
  the explicit dependency on the number of objectives $k$:
  \begin{equation}
    r_t \le 2 \beta_t \sqrt{k} \sqrt{\text{tr}(\Psi(\bs_t) \bV_t^{-1} \Psi(\bs_t)^\top)}.
  \end{equation}
  Summing over $T$ rounds and applying the Matrix Elliptical Potential Lemma
  \citep{abbasi2011improved}:
  \begin{equation}
    \sum_{t=1}^T r_t \le 2 \beta_t \sqrt{k} \sqrt{T \sum_{t=1}^T \text{tr}(\Psi(\bs_t) \bV_t^{-1} \Psi(\bs_t)^\top)}
    \le 2 \beta_t \sqrt{k} \sqrt{T \cdot 2 d_{\mathrm{eff}} \log\left(1 + \frac{T L^2}{d_{\mathrm{eff}}\lambda}\right)}.
  \end{equation}
  Substituting $\beta_t \approx \sqrt{d_{\mathrm{eff}}}$, we obtain
  $R_{pred} \le \wt O(d_{\mathrm{eff}} \sqrt{k T})$.
\end{proof}

\textbf{Comparison to Node-Wise Bound.} The standard bound
(Theorem~\ref{th:regret-lazy-graph-LinUCB}) scales with $\sum_{i=1}^k d_i$. In
the Factor model, the joint estimator scales with
$d_{\mathrm{eff}} \sqrt{k} = (\sum_{C} d_C) \sqrt{k}$. Consider a dense case
where $k$ nodes form a single clique of dimension $D$. The Node-wise bound
scales as $k D$. The Joint bound scales as $D \sqrt{k}$. Thus, the joint
estimator improves the regret by a factor of $\sqrt{k}$ in dense, highly
correlated systems,
see Figure~\ref{fig:clique_overlap}.
\ignore{ \textbf{Comparison to Node-Wise Bound.}  The
  standard bound (Theorem 3) scales with $\sum_{i=1}^k \sqrt{d_i}$. In the
  Factor model, $d_i \approx \sum_{C \ni i} d_C$.  Thus, the old bound is
  roughly $\sum_{i} \sqrt{\sum_{C \ni i} d_C}$. The new bound depends on
  $\sqrt{d_{\mathrm{eff}}} = \sqrt{\sum_{C} d_C}$.  Consider a graph where every
  node belongs to a clique of size $S$, and each clique has dimension $D$: (1)
  Node-wise (Theorem 3): $\sum_{i=1}^k \sqrt{D} \approx k \sqrt{D}$; (2) Joint
  (Theorem \ref{th:joint_regret}):
  $\sqrt{(\frac{k}{S} D) T} \cdot \sqrt{k} \approx \sqrt{k^2 D/S} \approx
  \frac{k}{\sqrt{S}} \sqrt{D}$. Thus, the joint estimator improves the regret by
  a factor of $\sqrt{S}$, where $S$ is the size of the overlapping cliques. This
  quantifies the gain from exploiting the graph's \emph{overlap structure}
  (cliques) rather than just its sparsity (see Figure~\ref{fig:clique_overlap}).
}

\section{Adversarial Setting}
\label{sec:adversarial}

We also extend our framework to the adversarial full-information setting, where
loss functions $\ell_t(\cdot)$ are arbitrary convex functions chosen by an
adversary. Unlike the stochastic setting, \emph{smart} determinant-based laziness
is not possible. To satisfy a strict switching budget $M$, we propose
\textsc{Randomized Lazy OGD}, which maintains a virtual OGD iterate but updates
the deployed state $\bs_t$ only with probability $p = M/T$.

We prove this achieves $\E[\Reg_T] \le \tilde{O}(D\,G\, T / \sqrt{M})$, where $D$ is
the diameter and $G$ is the Lipschitz constant (Theorem~\ref{th:lazy_ogd}),
quantifying the \emph{price of sparsity}: regret degrades from $\sqrt{T}$ to
$T/\sqrt{M}$. An extension for heterogeneous budgets $\bM = (M_1, \dots, M_k)$
is also provided. 

In many real-world scenarios, user feedback may not be stochastic but rather
driven by changing population drifts or even adversarial behavior. We now extend
our framework to the adversarial setting.

\subsection{Problem Setup}

In this setting, the loss functions $\ell_t(\bs)$ are arbitrary convex functions
chosen by an adversary. The objective is to minimize the regret with movement
costs:
\begin{equation}
  \Reg_T = \sum_{t=1}^T \left( \ell_t(\bs_t)
    + \lambda \| \bs_t - \bs_{t-1} \|_1 \right) - \sum_{t=1}^T \ell_t(\bs^*).
  \label{eq:adv_regret}
\end{equation}
It is established that standard Online Gradient Descent (OGD) with learning rate
$\eta \propto 1/\sqrt{T}$ achieves the optimal $\wt O(\sqrt{T})$ regret for this
objective \citep{cesa2013online}. However, standard OGD updates the system state
$\bs_t$ at \emph{every} time step. In large-scale production systems, such
high-frequency updates ($T$ switches) are often operationally prohibitive.

Full Information Setting.
Unlike the stochastic bandit setting in Section~\ref{sec:algorithm}, we assume
here a \emph{Full Information} (Online Convex Optimization) setting where the
learner observes the gradient $\nabla \ell_t$ after every round. We reference
bandit lower bounds (specifically the $T^{2/3}$ rate) in the analysis below
solely for context: they serve as a useful baseline to illustrate how strictly
limiting the switching budget $M$ degrades the regret of a full-information
algorithm to rates typically associated with limited feedback.

Sparsity Constraint.
We consider a practical setting where the learner is constrained by a
\emph{switching budget} $M \ll T$. The goal is to minimize the regret in
\eqref{eq:adv_regret} subject to the constraint that the expected number of
updates is at most $M$.

\subsection{Algorithm: Randomized Lazy OGD}

To address the need for sparsity, we use a \emph{Randomized Lazy OGD}
strategy. The core idea is to decouple the learning process from the state
actuation. We maintain a virtual OGD iterate $\bw_t$ that updates at every step
to track the loss landscape, but we only update the actual deployed state
$\bs_t$ with probability $p$.

\begin{algorithm}[H]
   \caption{Randomized Lazy OGD}
   \label{alg:lazy_ogd}
\begin{algorithmic}[1]
  \STATE \textbf{Input:} Learning rate $\eta$, Switching probability $p$,
  Initial state $\bs_0 \in \sS$.
  \STATE Initialize virtual state $\bw_1 = \bs_0$.
   \FOR{$t=1$ to $T$}
   \STATE Play $\bs_t$ and observe loss function $\ell_t(\cdot)$
   (and gradient $\nabla \ell_t(\bs_t)$).
       \STATE \textbf{Virtual Update:}
       \STATE $\bw_{t+1} = \Pi_{\sS} \left( \bw_t - \eta \nabla \ell_t(\bs_t) \right)$
       \STATE \textbf{Lazy State Update:}
       \STATE Sample $z_t \sim \text{Bernoulli}(p)$.
       \IF{$z_t = 1$}
           \STATE $\bs_{t+1} = \bw_{t+1}$
       \ELSE
           \STATE $\bs_{t+1} = \bs_t$
       \ENDIF
   \ENDFOR
\end{algorithmic}
\end{algorithm}

Unlike the stochastic setting where \emph{smart} data-dependent updates are possible,
the adversarial setting fundamentally limits our ability to be lazy. We analyze
a randomized update schedule not as a novel algorithm per se, but to quantify
the unavoidable regret trade-off between stability and accuracy. This mirrors
the general trade-off between regret and movement costs (or ``service costs'') in
online convex optimization \citep{andrew2013tale}. Theorem~\ref{th:lazy_ogd}
establishes the \emph{price of sparsity}, showing how regret degrades from
$\sqrt{T}$ as the switching budget tightens, and recovering the classical
$T^{2/3}$ rate in the regime of very small switching budgets.

\subsection{Regret guarantees}

The following theorem characterizes the trade-off between the sparsity budget
$M$ and the achievable regret.

\begin{restatable}[Regret with Sparsity Budget]{theorem}{LazyOGD}
  \label{th:lazy_ogd}
  Let the loss functions $\ell_t$ be convex, $G$-Lipschitz, and defined over a
  domain of diameter $D$ (where $D$ bounds both the $\ell_1$ and $\ell_2$
  diameters). For any switching budget $M \in [1, T]$, running Algorithm
  \ref{alg:lazy_ogd} with switching probability $p = M/T$ and learning rate
  $\eta = \frac{D}{G T}\sqrt{\frac{M}{2}}$ guarantees:
  \begin{equation}
    \E[\Reg_T] \le \lambda D M + O\left( DG \frac{T}{\sqrt{M}} \right).
  \end{equation}
\end{restatable}

\begin{proof}
  We analyze the expected regret of \textsc{Randomized Lazy OGD}. Let
  $f_t(\bs) = \ell_t(\bs)$. The total regret decomposes into the loss regret and
  the movement cost.
  \[
    \E[\Reg_T] = \underbrace{\E\left[ \sum_{t=1}^T (f_t(\bs_t) - f_t(\bs^*))
      \right]}_{\text{loss regret}} + \underbrace{\E\left[ \sum_{t=1}^T \lambda
        \| \bs_t - \bs_{t-1} \|_1 \right]}_{\text{movement cost}}.
  \]
  The state $\bs_t$ changes only when the Bernoulli variable $z_{t-1} =
  1$. Thus, we have
  \[
    \E[ \| \bs_t - \bs_{t-1} \|_1 ] = \Pr(z_{t-1}=1) \cdot \E[ \| \bw_t -
    \bs_{t-1} \|_1 ] \le p D.
  \]
  Summing over $T$ steps gives $\E[R_{\text{move}}] \le p T \lambda D$.  Now, to
  bound the loss regret, we use the convexity of $f_t$.
  \begin{equation}
    f_t(\bs_t) - f_t(\bs^*) \le \tri{\nabla f_t(\bs_t), \bs_t - \bs^*}
    = \underbrace{\tri{\nabla f_t(\bs_t), \bs_t - \bw_t}}_{\text{Coupling Error}}
    + \underbrace{\tri{\nabla f_t(\bs_t), \bw_t - \bs^*}}_{\text{Virtual Regret}}.
  \end{equation}
  The sequence $\bw_t$ follows standard Online Gradient Descent on the
  linearized loss functions $\tilde{f}_t(\bw) = \tri{\nabla f_t(\bs_t),
    \bw}$. Using the standard OGD bound:
  \begin{equation}
    \sum_{t=1}^T \tri{\nabla f_t(\bs_t), \bw_t - \bs^*} \le \frac{D^2}{2\eta}
    + \frac{\eta}{2} \sum_{t=1}^T \|\nabla f_t(\bs_t)\|^2 \le \frac{D^2}{2\eta}
    + \frac{\eta T G^2}{2}.
  \end{equation}
  For the Coupling Error, using the Lipschitz property
  ($ \|\nabla f_t(\bs_t)\| \le G $) and the geometric distribution of the
  staleness ($t - \tau(t)$) with mean $1/p$:
  \begin{equation}
    \E\left[ \tri{\nabla f_t(\bs_t), \bs_t - \bw_t} \right]
    \le G \, \E[ \|\bs_t - \bw_t\| ] \le \frac{\eta G^2}{p}.
  \end{equation}
  Summing over $T$ gives $\E[\text{Coupling Error}] \le \frac{T \eta
    G^2}{p}$. Combining the bounds, the total expected regret is upper bounded
  by:
  \begin{equation}
    \E[\Reg_T]
    \le \underbrace{p T \lambda D}_{\text{Movement Cost}}
    + \underbrace{\frac{D^2}{2\eta}
      + \frac{\eta T G^2}{2}}_{\text{Virtual OGD Regret}}
    + \underbrace{\frac{T \eta G^2}{p}}_{\text{Coupling Error}}.
    \label{eq:regret_decomp_budget}
  \end{equation}
  We are given a switching budget $M$. To ensure the expected number of switches
  is at most $M$, we set the switching probability to:
  \[
    p = \frac{M}{T}.
  \]
  Substituting $p=M/T$ into \eqref{eq:regret_decomp_budget}:
  \[
    \E[\Reg_T] \le M \lambda D + \left( \frac{D^2}{2\eta} + \frac{\eta T G^2}{2}
    \right) + \frac{T^2 \eta G^2}{M}.
  \]
  Assuming the sparsity regime $M \ll T$, the term $\frac{T^2 \eta G^2}{M}$
  dominates $\frac{\eta T G^2}{2}$. We simplify the objective to balancing the
  Virtual Regret (specifically the $D^2/2\eta$ term) against the Coupling Error:
  \[
    \text{Bound}(\eta) \approx M \lambda D + \frac{D^2}{2\eta} + \frac{T^2 \eta
      G^2}{M}.
  \]
  Minimizing the $\eta$-dependent terms yields the optimal learning rate:
  \[
    \frac{D^2}{2\eta^2} = \frac{T^2 G^2}{M} \implies \eta^* = \frac{D}{TG}
    \sqrt{\frac{M}{2}}.
  \]
  Substituting $\eta^*$ back into the bound:
  \[
    \frac{D^2}{2\eta^*} + \frac{T^2 \eta^* G^2}{M} = \frac{D^2 T G}{2 D
      \sqrt{M/2}} + \frac{T^2 G^2}{M} \frac{D}{TG} \sqrt{\frac{M}{2}} = \sqrt{2}
    \frac{DTG}{\sqrt{M}}.
  \]
  Thus, the final bound is:
  \[
    \E[\Reg_T] \le M \lambda D + O\left( DG \frac{T}{\sqrt{M}} \right).
  \]
  This completes the proof.
\end{proof}

This result quantifies the \emph{price of sparsity}:
\begin{itemize}

\item Unconstrained ($M=T$): The algorithm reduces to standard Online Gradient
  Descent. While this yields optimal prediction regret ($O(\sqrt{T})$), the
  bound on movement costs becomes linear ($O(T)$).
\item Optimal Trade-off ($M = \Theta(T^{2/3})$): We recover a total regret rate
  of $O(T^{2/3})$. This matches the minimax rate for bandit problems with
  switching costs.
\item High Sparsity ($M = \sqrt{T}$): The regret degrades to $O(T^{3/4})$,
  illustrating the unavoidable cost of freezing the state in an adversarial
  environment.
\end{itemize}

Unlike the stochastic setting where ``smart'' data-dependent updates are
possible, the adversarial setting limits our ability to exploit stable epochs.
We analyze a randomized update schedule to quantify the regret trade-off
inherent to \emph{lazy} gradient-based strategies.  While perturbation-based
methods \citep{kalai2005efficient} can achieve $\wt O(\sqrt{T})$ regret with
$\wt O(\sqrt{T})$ switches for linear losses, they often require full
optimization oracles.  In contrast, our approach applies efficiently to
\emph{general convex losses}.  Theorem~\ref{th:lazy_ogd} establishes the
\emph{price of sparsity} for this lazy strategy, showing how regret degrades
from $\sqrt{T}$ as the switching budget tightens.  Note that, under the optimal
switching budget ($M \approx T^{2/3}$), this full-information strategy recovers
the $O(T^{2/3})$ rate, mirroring the minimax lower bound for bandits with
switching costs \citep{dekel2014bandits}.  \ignore{ It is worth noting the
  distinction between this adversarial strategy and the stochastic strategy
  proposed in Section \ref{sec:algorithm}. In the adversarial setting, we use a
  fixed randomized $p$ to satisfy a global budget $M$. In contrast, our
  stochastic algorithm, \textsc{LazyGraphLinUCB}, uses a \emph{time-varying,
    data-dependent} schedule ($\det(\bV_{t,i}) > 2 \det(\bV_{\mathrm{last}, i})$). The
  stochastic setting allows for ``smart'' laziness (waiting for information),
  whereas the adversarial setting requires ``randomized'' laziness (to satisfy
  the budget without being predictable).}  The parameter $p$ explicitly controls
the trade-off between the system's stability and its ability to track the
optimal parameters. This can be viewed as an optimization problem. From the
proof of Theorem \ref{th:lazy_ogd}, the upper bound on the regret is
approximately:
\begin{equation}
  \text{Bound}(p) \approx \underbrace{p T \lambda D}_{\text{Movement Cost}}
  + \underbrace{\frac{\eta T G^2}{p}}_{\text{Coupling Error}} + C.
\end{equation}
The first term reflects the penalty for instability: a higher $p$ leads to more
frequent switches. The second term reflects the \emph{coupling error}, the
divergence between the active system state $\bs_t$ and the virtual learner
$\bw_t$. Minimizing this bound with respect to $p$ (for a fixed learning rate
$\eta$) yields the optimal fixed switching probability:
\begin{equation}
  p^* \propto \sqrt{\frac{\eta G^2}{\lambda D}}.
\end{equation}
This result offers a principled heuristic: the switching frequency should be
inversely proportional to the square root of the movement penalty $\lambda$.

\subsection{Extension: Heterogeneous Stability Budgets}
\label{sec:adversarial_extension}

The standard \textsc{Randomized Lazy OGD} applies a uniform switching
probability $p$ across all dimensions. However, in the conflicting objectives
setting, different criteria often possess different stability requirements. For
example, a fairness threshold might require strict stability (low update
frequency) to avoid public outcry, while an internal mixing rate might tolerate
frequent updates.  We can extend our framework to support \emph{heterogeneous
  sparsity budgets}. Let $\bs$ be decomposed into $k$ independent components (or
blocks) with distinct switching budgets $\bM = (M_1, \dots, M_k)$. We assume the
loss function decomposes additively,
$\ell_t(\bs) = \sum_{i=1}^k \ell_{t,i}(s_i)$, which is consistent with the
graph-structured formulation in Section~\ref{sec:formulation}
(where blocks correspond to disjoint cliques).

\textbf{Algorithm: Component-Wise Lazy OGD.}  We run $k$ parallel
instances of the virtual update logic. For each component $i$:\\
1. Maintain virtual iterate $w_{t,i}$.\\
2. Sample switching variable $z_{t,i} \sim \text{Bernoulli}(p_i)$
independent of other components.\\
3. Update deployed state $s_{t,i} \leftarrow w_{t,i}$ only if $z_{t,i}=1$.

\textbf{Theoretical Guarantee.}  By setting $p_i = M_i / T$ and tuning
component-specific learning rates $\eta_i$, we achieve a refined regret bound
that sums the localized costs rather than scaling with the global dimension.
\begin{corollary}[Regret with Heterogeneous Budgets]
  Let the loss $\ell_t$ decompose into $k$ components with diameters $D_i$ and
  Lipschitz constants $G_i$. Given a vector of budgets $\bM$, the expected
  regret is bounded by:
  \begin{equation}
    \E[\Reg_T] \le \sum_{i=1}^k \left( \lambda D_i M_i
      + O\left( D_i G_i \frac{T}{\sqrt{M_i}} \right) \right).
  \end{equation}
\end{corollary}

This extension is significant because it allows the system designer to
\emph{spend} sparsity where it is most needed. Unlike the global strategy which
pays a regret penalty governed by the worst-case stability constraint
($M_{min}$), the component-wise strategy isolates the cost of stability to the
specific criteria that require it.

\begin{figure*}[t]
  \centering
  \begin{minipage}[t]{0.48\textwidth}
    \centering
    \includegraphics[width=\textwidth]{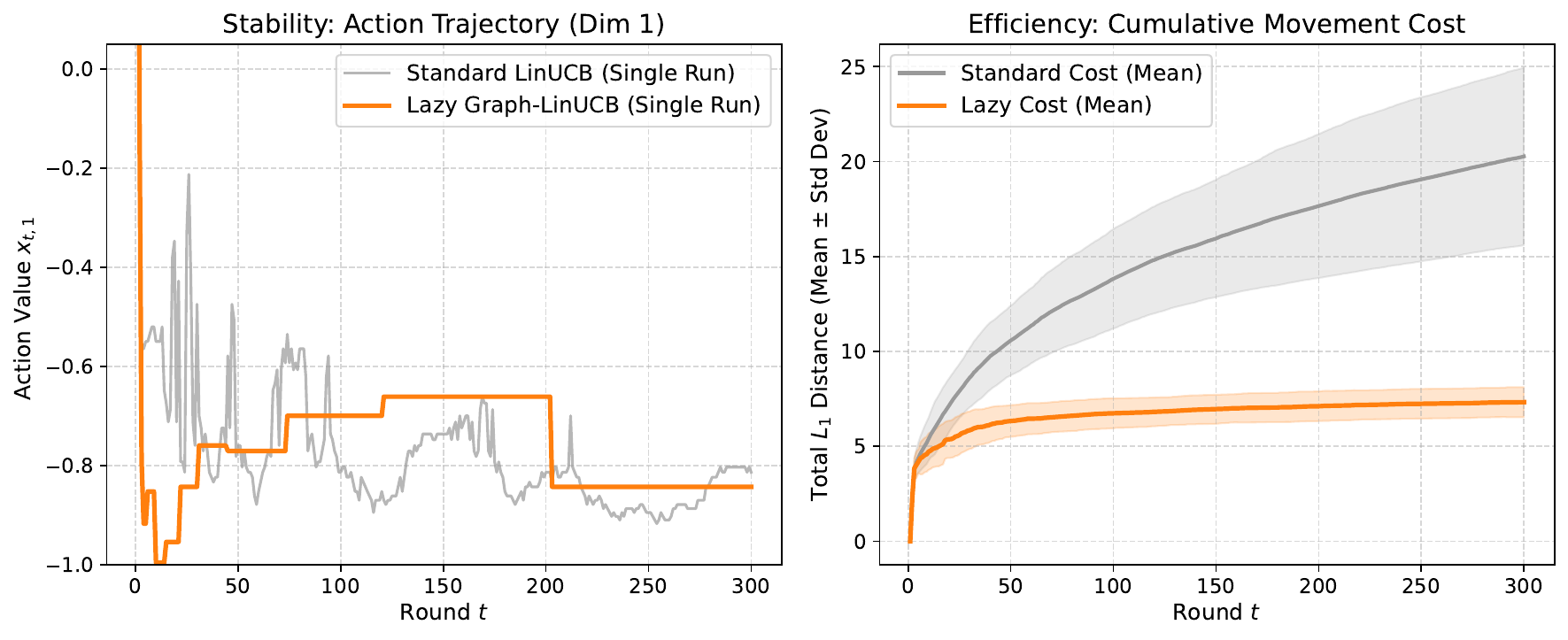}
    \vskip -.1in
    \caption{\textbf{Impact of Lazy Updates.} (Left) The trajectory of the
      selected action $s_t$ (dimension 1). Standard LinUCB (grey) exhibits
      constant jitter, while \textsc{LazyGraphLinUCB} (orange) freezes the
      action. (Right) Cumulative movement cost averaged over 100 runs. The
      standard algorithm accumulates linear cost ($\approx 22$), while the lazy
      strategy saturates ($\approx 8$).}
    \label{fig:synthetic}
        
    \end{minipage}
    \hfill
    \begin{minipage}[t]{0.5\textwidth}
      \centering
      \includegraphics[width=\textwidth]{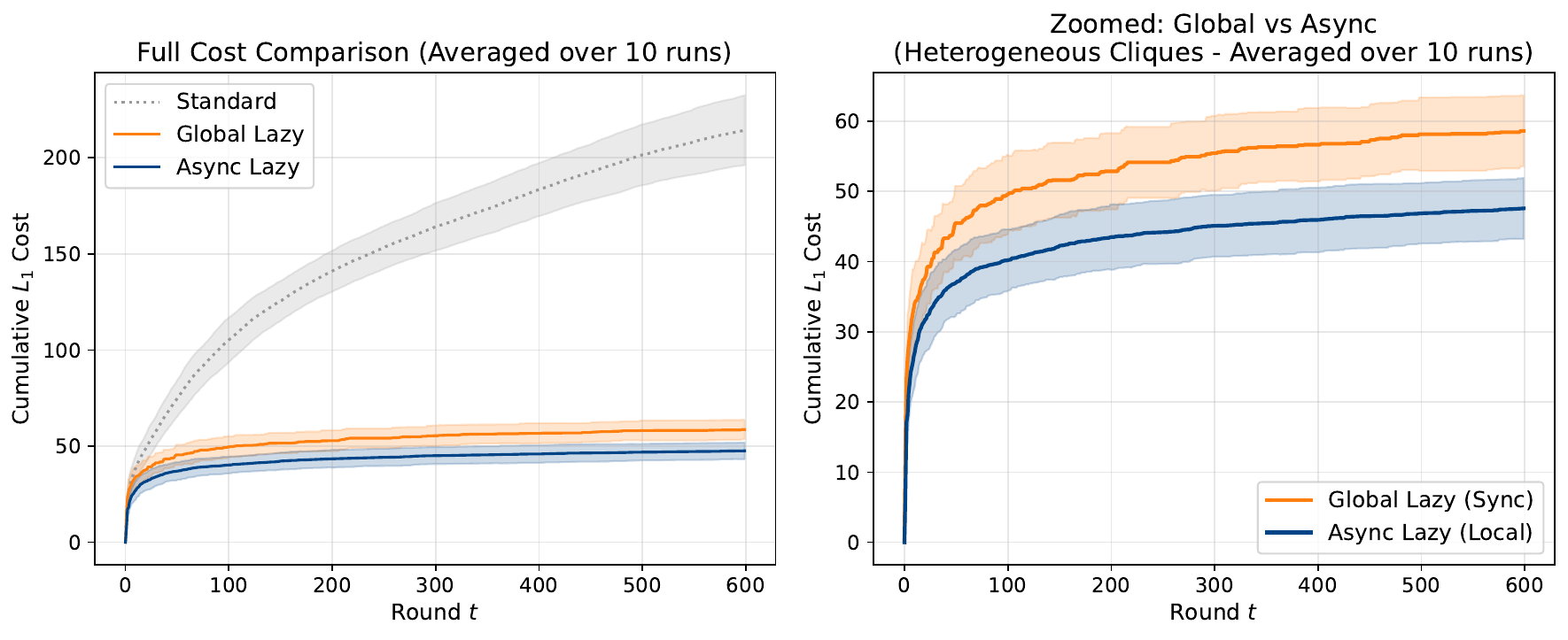}
      \vskip -.1in \caption{\textbf{Scalability of Asynchronous Updates.} (Left)
        Comparison against the standard baseline. (Right) Zoomed comparison of
        lazy strategies. \textsc{Global Lazy} (orange) suffers from
        synchronization penalties, while \textsc{Async Lazy} (blue) decouples
        components, reducing movement costs by $\mathbf{3.1\times}$ (see
        Table~\ref{tab:trigger_analysis}).}
      \label{fig:graph_experiment}
    \end{minipage}
    \vskip -.2in
  \end{figure*}

  \section{Numerical Illustrations}
  \label{sec:experiments}

  To validate our theoretical findings, we conducted two synthetic experiments
  and one real-world application. Our primary contribution is theoretical; these
  experiments are not large-scale benchmarks, but numerical illustrations
  isolating and validating specific mechanisms predicted by our analysis: the
  reduction of jitter and the elimination of synchronization overhead.

  \subsection{Single-Clique Stability Analysis}
  \label{sec:exp_single_clique}

  We first compared \textsc{LazyGraphLinUCB} against standard LinUCB on a simple
  $d=2$ dimensional linear bandit problem over $T=300$ rounds. The goal was to
  minimize a noisy linear loss function
  $\ell_t(s) = \langle s, \theta^* \rangle + \mathcal{N}_t(0, \sigma^2)$. The
  parameters used were $\lambda_{\mathrm{reg}}=0.1$, $\beta=0.5$ for the
  exploration term, and noise parameter $\sigma=0.5$. Both algorithms rely on
  the same regularized least-squares estimator; the critical difference lies in
  the update schedule. Standard LinUCB updates its policy at every step to the
  minimizer of the current LCB, whereas \textsc{LazyGraphLinUCB} only updates
  when the determinant of the covariance matrix doubles. We note that the
  cumulative prediction regret for both algorithms was statistically
  indistinguishable, confirming that the stability gains did not come at the
  cost of accuracy.

  \textbf{Stability vs.\ Jitter.} Figure~\ref{fig:synthetic} (Left) illustrates
  the core mechanism of our algorithm. Because the estimator $\h \theta_t$
  fluctuates with every noisy sample, the standard LinUCB action $s_t$ jitters
  constantly. In contrast, \textsc{LazyGraphLinUCB} freezes the action for
  stable epochs. The initial fluctuations in the first $\approx 20$ rounds
  correspond to the rapid shrinkage of the confidence ellipsoid (where the
  determinant doubles frequently). However, as $t$ increases, the updates become
  exponentially rare.

  \textbf{Cost Reduction.} Figure~\ref{fig:synthetic} (Right) confirms our
  regret bounds. The standard algorithm incurs linear movement cost
  ($\Omega(T)$), making it unusable in systems with high switching
  penalties. Our lazy algorithm maintains logarithmic movement cost
  ($O(\log T)$) while provably retaining the optimal $\wt O(\sqrt{T})$
  prediction accuracy. (Note: In our experiments, the cumulative prediction
  regret of the Lazy algorithm was statistically indistinguishable from the
  Standard algorithm; we omit the redundant learning curves for brevity.)

  \subsection{Heterogeneous Graph Scalability}
  \label{sec:exp_hetero_graph}

  To validate the benefits of our graph-structured approach, we simulated a
  larger system with $k = 10$ independent criteria (cliques), each of dimension
  $d=2$. To mirror real-world conditions where different objectives learn at
  different rates (e.g., a noisy \emph{CTR} objective vs.\ a stable \emph{Latency}
  objective), we introduced heterogeneity by assigning a random stiffness
  parameter $\lambda_{\mathrm{reg},i} \in [0.01, 2.0]$ to each clique. The
  higher the $\lambda_{\mathrm{reg},i}$, the longer it will take for a clique to
  reach its triggering threshold.

  The results in Figure~\ref{fig:graph_experiment} highlight the critical
  importance of the asynchronous extension
  (Section~\ref{sec:async_extension}). Under \textsc{Global Lazy}, all cliques
  are forced to update whenever \emph{any} criterion triggers, so even
  stiff (slow-learning) cliques update frequently despite gathering
  little new information. \textsc{Async Lazy} (Theorem~\ref{th:async_regret}) eliminates
  this waste. As shown in Table~\ref{tab:trigger_analysis}, stable components
  (e.g., Clique 9) incur a $3.7\times$ waste factor under the global strategy as it is updated 41 times vs.\ 11 times under the asynchronous strategy;
  decoupling updates reduces total system-wide moves by isolating changes to the
  relevant local neighborhood.

\begin{table}[h]
  \centering
  \caption{\textbf{Mechanism of Efficiency (Excerpt).} Comparison of update
    counts for representative cliques. Volatile cliques update frequently in
    both strategies; stable cliques freeze in Async but are forced to move in
    Global. See Appendix~\ref{app:trigger_analysis} for the full
    breakdown of all 10 cliques.}
  \label{tab:trigger_analysis}
  \vspace{0.1cm}
  \begin{small}
    \begin{tabular}{lcccc}
      \toprule
      \textbf{Clique Type} & \textbf{Stiffness,} & \textbf{Global,}
      & \textbf{Async,} & \textbf{Waste} \\
       & \textbf{$\lambda_{\mathrm{reg}}$} & \textbf{moves} & \textbf{moves} & \textbf{factor}\\
      \midrule
      Volatile (ID 4) & 0.19 & 41 & 15 & 2.7$\times$ \\
      Stable (ID 9) & 1.40 & 41 & 11 & \textbf{3.7$\times$} \\
      \midrule
      \textbf{Total (All 10)} & - & \textbf{410} & \textbf{131}
                       & \textbf{3.1$\times$} \\
      \bottomrule
    \end{tabular}
  \end{small}
\end{table}

\subsection{Real-World Data: Fairness Threshold Tuning}
\label{sec:exp_adult}

To validate our framework in a realistic setting, we applied it to a fairness
tuning task using the \emph{Adult Census dataset}. We trained a logistic
regression classifier to predict income $> \$50$K and formulated the problem of
tuning the decision threshold $\tau \in [0, 1]$ as a bandit problem with
inherent conflicting objectives: maximizing accuracy of the income prediction
while minimizing the Equal Opportunity gap. Specifically, we scalarize the
global objective as the sum of the error rate and the fairness violation:
\begin{equation*}
  L(\tau) = (1-\text{Accuracy}(\tau))
  +|\text{TPR}_{\text{male}}(\tau) - \text{TPR}_{\text{female}}(\tau)|.
\end{equation*}
To capture the non-linear dependence of these metrics on the threshold, we use a
quadratic feature map $\phi(\tau) = [1, \tau, \tau^2]^\top$.  We initialized the
threshold at a low value ($\tau=0.2$) to observe the learning trajectory as the
agent discovers the high-threshold region ($\tau \approx 0.9$) required for
fairness. We used $\lambda_{\mathrm{reg}} = 5$ to initialize the covariance matrix
$\bV$. To simulate realistic feedback variability, we added zero-mean Gaussian
noise ($\sigma=0.05$) to the observed losses. While this experiment focuses on a
single parameter ($k=1$) and thus does not exercise the graph sparsity
mechanisms (illustrated in Section~\ref{sec:exp_hetero_graph}), it serves to
empirically validate the efficiency and stability of the continuous-state
\textsc{Lazy} formulation on a realistic, non-convex loss surface, complementing
the structural verification in the synthetic experiments.

\textbf{Results.} Figure~\ref{fig:adult_experiment} (Left) shows the mean
threshold trajectory averaged over 10 random seeds. Both algorithms successfully
converge to the fair region ($\tau \approx 0.9$). However, the standard LinUCB
algorithm (grey) exhibits high variance and frequent oscillations due to
sensitivity to batch noise. In contrast, \textsc{LazyGraphLinUCB} (blue)
demonstrates a stable learning curve with tight confidence bands, filtering out
transient noise.

Figure~\ref{fig:adult_experiment} (Right) demonstrates that the stability
translates directly to operational efficiency. In this context, we define
operational efficiency as the minimization of \emph{policy churn} (unnecessary
parameter updates that trigger system maintenance costs (e.g., cache
invalidation or downstream re-computations) without yielding performance
gains). The Lazy strategy reduces the average cumulative movement cost by a
factor of \textbf{5.7$\times$} (14.0 vs 2.5) compared to the standard LinUCB greedy
approach. The error bars confirm that this gain is statistically significant and
robust across different random initializations.

\section{Limitations}
\label{sec:limitations}

\begin{figure}[t]
  \centering
  \includegraphics[width=0.48\textwidth]{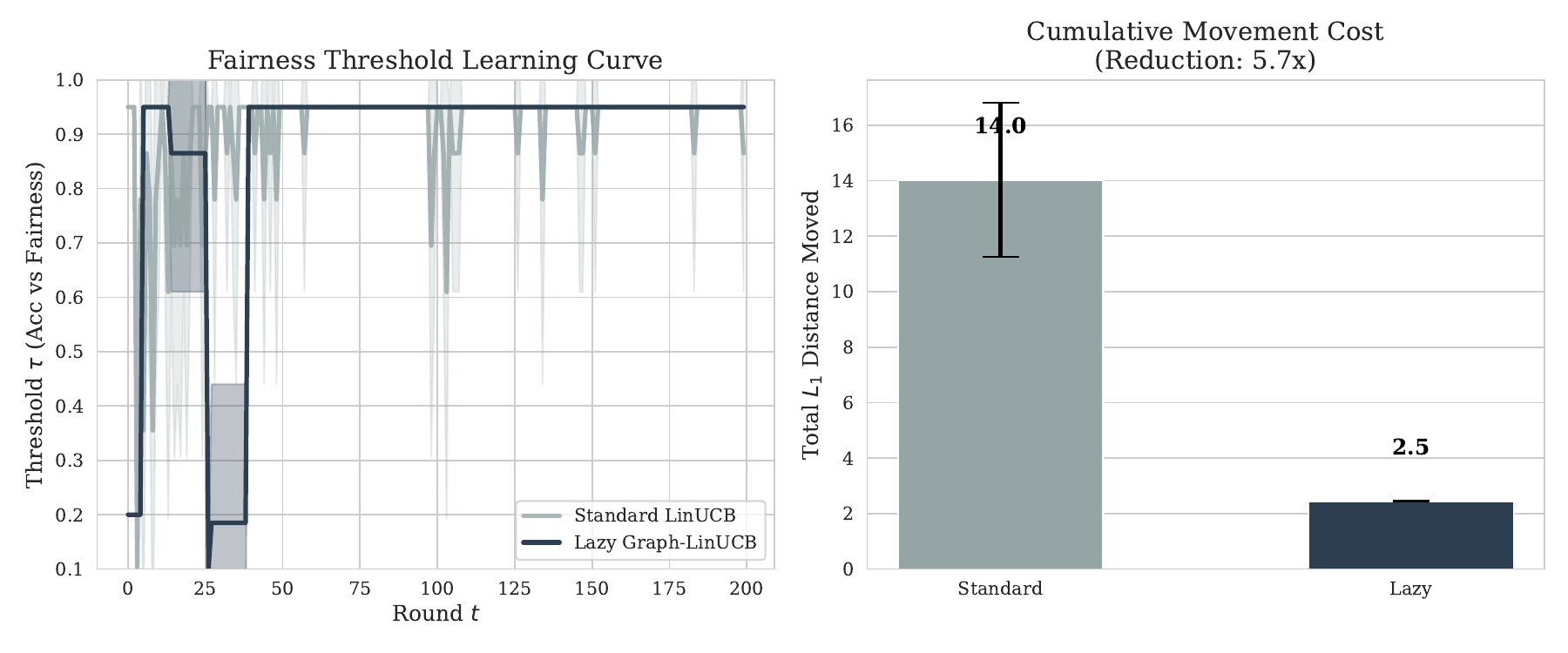}
  \caption{\textbf{Fairness Threshold Tuning on Adult Dataset.} (Left) Mean
    learning trajectory ($\pm 1$ std.\ dev.) over 10 independent trials. The
    Lazy algorithm (blue) consistently tracks the optimal threshold with
    significantly reduced variance compared to the standard LinUCB approach
    (grey). (Right) Cumulative movement cost. The Lazy strategy significantly
    reduces operational churn, achieving a $\approx 5.7\times$ reduction in
    total movement cost.}
  \vskip -.2in
  \label{fig:adult_experiment}
\end{figure}

We discuss the main assumptions and limitations of our framework.  First,
\emph{non-convex LCB optimization}: minimizing the LCB objective
\eqref{eq:10} is non-convex in general; in dense graphs where the local
dimension $d_i$ is large, even approximate solvers may be costly. The lazy
schedule amortizes this cost (at most $O(d_{\max} k \log T)$ solves over $T$
rounds), and Section~\ref{sec:async_extension} restricts solutions to
low-dimensional local neighborhoods.  Second, \emph{stationarity of the
  dependency graph}: the adaptive graph learning mechanism
(Section~\ref{sec:graph_learning}) assumes the underlying dependency structure
is fixed. If dependencies drift over time, a sliding-window or discounted
covariance approach would be needed.  Third, \emph{Assumption A in the
  asynchronous algorithm}: the PL condition (formally justified in
Lemma~\ref{lemma:assumption-a-justified}) holds when $\lambda_{\mathrm{reg}}$ is
sufficiently large; in practice this trades exploration for convergence speed.
Finally, our experiments focus on synthetic benchmarks and a single-parameter
real-world task ($k=1$); large-scale evaluation of the graph mechanisms on real
multi-criterion systems is left to future work.

\section{Conclusion}
\label{sec:conclusion}

We presented a comprehensive extension of the feedback-driven competing
objectives framework to continuous state spaces. Replacing binary states with a
continuous domain and modeling local dependencies via graph-structured linear
function approximation yields a more realistic model for tuning hyperparameters
such as fairness thresholds and diversity rates. We addressed stability via
movement costs, showing that a \emph{lazy} update strategy achieves sub-linear
regret in both stochastic and adversarial settings. Furthermore, the dependency
graph is not merely a constraint but a source of efficiency: an
\emph{asynchronous} schedule minimizes movement in sparse systems, an
\emph{adaptive} algorithm learns dependencies from data, and a
\emph{factor-graph} decomposition tightens regret bounds.  Future work includes
large-scale empirical evaluation and extending the factor-graph analysis to
non-linear function approximators.

\ignore{
\newpage
\section*{Impact Statement}
This work contributes to the reliability of machine learning systems by
providing rigorous methods for multi-objective optimization. In particular, our
framework allows for the stable tuning of sensitive parameters, such as fairness
thresholds, potentially reducing the risk of erratic system behavior that could
negatively impact users. We do not foresee immediate negative societal
consequences, though as with any optimization tool, the outcome depends on the
validity of the chosen objectives.
}

\section*{Acknowledgments}

YM's work is supported in part by European Research Council (ERC) under the
European Union’s Horizon 2020 research and innovation program (grant agreement
No. 882396), by the Israel Science Foundation, the Yandex Initiative for Machine
Learning at Tel Aviv University and a grant from the Tel Aviv University Center
for AI and Data Science (TAD).

\newpage
\bibliographystyle{abbrvnat}
\bibliography{cfair}

\newpage
\appendix
\onecolumn

\renewcommand{\contentsname}{Contents of Appendix}
\tableofcontents
\addtocontents{toc}{\protect\setcounter{tocdepth}{3}}
\clearpage

\section{Extended Related Work}
\label{app:extended_related}

Our work extends and connects to a broad literature on multi-objective
optimization, algorithmic fairness, and online learning. We survey the most
relevant threads below, explaining how they motivate our continuous-state
formulation and contrast with it.

\subsection{Multi-Objective Optimization and Competing Criteria}

There is a vast body of literature on optimizing multiple metrics or objectives
simultaneously. Works by \citet{mohri2019agnostic} and
\citet{cortes2020agnostic} design ``agnostic'' algorithms that compete with any
linear or convex combination of a fixed set of base objectives
$L_1, \dots, L_k$. Another major line of research seeks Pareto-optimal solutions
\citep{jin2008pareto, sener2018multi, shah2016pareto, marler2004survey}. In
multi-task learning, objective conflicts are commonly addressed via gradient
manipulation or constrained optimization \citep{sener2018multi}. The key
difference with our framework is that we do \emph{not} commit to a fixed linear
combination or Pareto front; instead, we learn from user feedback which
configuration of the continuous parameters best satisfies the objectives.

A closely related and direct predecessor of the current paper is the work of
\citet*{AwasthiCortesMansourMohri2024}, which introduced the feedback-driven
multi-criteria model using an incompatibility graph and binary states. The
present work extends that model to continuous state spaces, replaces binary
``fix/unfix'' actions with continuous parameter tuning, and replaces MDP state
enumeration with graph-structured linear bandit learning. The journal version of
that work \citep{AwasthiCortesMansourMohri2026} further develops the adversarial
setting, which we adapt in Section~\ref{sec:adversarial}.

\subsection{Algorithmic Fairness: Classification, Ranking, and Clustering}

The problem of satisfying multiple fairness constraints has been extensively
studied. In the classification setting, \citet{agarwal2018reductions} reduce the
fair classification problem to a sequence of cost-sensitive classification
problems, and \citet{cotter2018training, CotterJiangSridharan2019} use
two-player game formulations to handle non-convex constraints. These works
typically tailor algorithms to specific group fairness notions (e.g.,
demographic parity, equal opportunity) as fixed binary constraints.  In
contrast, our framework treats \emph{all} criteria uniformly and optimizes their
continuous thresholds jointly.

Fairness in ranking \citep{celis_fair_ranking, beutel2019fairness,
  narasimhan2019pairwise} and clustering \citep{fair_clustering_Nips2017,
  schmidt2018fair, backurs2019scalable} has also received significant attention,
typically with fixed fairness criteria. Work on individual fairness
\citep{dwork2012fairness, kearns2019average} considers settings where fairness
is a distributional condition over individuals.

The motivation for our continuous-state model is precisely that real-world
fairness interventions involve \emph{thresholds}—e.g., the maximum allowed gap
in True Positive Rates between demographic groups—rather than binary
constraints. Setting $\tau = 0.05$ vs.\ $\tau = 0.10$ produces meaningfully
different system behavior, and the optimal threshold must be learned from
deployed-system feedback. This is the regime our model directly addresses.

\subsection{Inherent Tension Between Multiple Metrics}

Several works demonstrate that multiple ``reasonable'' fairness criteria are
fundamentally incompatible.  \citet{kleinberg2017} prove that calibration and
equal opportunity cannot be simultaneously satisfied except in degenerate cases.
\citet{feller2016computer} discuss the COMPAS recidivism tool, where optimizing
one group's false positive rate necessarily worsens another
metric. \citet{menon2018cost} study the trade-off between accuracy and false
positive/negative rates.

Our dependency graph $\cG$ directly encodes these incompatibility structures: an
edge $(i, j) \in E$ means that optimizing criterion $i$ affects criterion $j$,
making the pair impossible to independently optimize. This graph-theoretic
encoding allows us to leverage sparse incompatibility structure for efficient
continuous optimization, generalizing the incompatibility graphs of
\citet*{AwasthiCortesMansourMohri2024} to the continuous-threshold setting.

\subsection{Long-Term Feedback Dynamics and MDP-Based Fairness}

Several works study the long-term consequences of optimizing multiple
conflicting criteria. \citet{liu2018delayed} show that minimizing a constrained
loss to equalize certain criteria can \emph{increase} disparate impact in the
long run. \citet{hashimoto2018fairness} develop algorithms for repeated loss
minimization that minimize such disparate impact. These works highlight that the
interaction between objectives plays out over time, motivating a dynamic,
feedback-driven framework rather than static optimization.

Closer to our MDP-based setting, \citet{jabbari2017fairness} study reinforcement
learning in MDPs subject to fairness constraints: the algorithm may not take
action $a$ over $a'$ if the long-term reward of $a$ is lower than $a'$. They
show that finding a near-optimal policy satisfying this criterion requires time
exponential in the state-space size—underscoring the importance of
\emph{structure} (our dependency graph) for efficient learning.
\citet{doroudi2017importance} show that off-policy importance sampling can
violate natural fairness criteria, and present corrective algorithms.

Our work differs from this line in a crucial respect: we do \emph{not} commit to
a fixed definition of quality or a fairness metric. The states in our MDP
correspond to \emph{configurations of continuous parameters}, and the graph
encodes which parameters interact. This makes our framework applicable to
arbitrary criteria (fairness, diversity, latency, revenue) without
specialization, and our algorithms are agnostic to the definition of the
criteria themselves.

\subsection{Fairness in Online Learning and Bandits}

Fairness constraints in online and bandit settings have been studied by
\citet{joseph2016fairness} (meritocratic fairness in classic and contextual
bandits), \citet{gillen2018online} (online learning with an unknown fairness
metric), and \citet{LiuRadanovicDimitrakakisMandalParkes2017} (calibrated
fairness in bandits). These works define specific fairness conditions and aim to
satisfy them while minimizing regret. Our framework generalizes beyond fairness
bandits by treating each criterion (including fairness) as a competing linear
objective with a continuous threshold, and by explicitly modeling the
\emph{movement costs} incurred when thresholds are adjusted.

\subsection{Monitoring, Auditing, and Feedback Mechanisms}

A key assumption of our stochastic model is that the algorithm can observe the
losses associated with different criteria at each time step. In practice, this
relates to the problem of monitoring and auditing deployed systems. There has
been considerable work on developing probabilistic verification of fairness
properties \citep{bastani2019probabilistic}, AI fairness toolkits
\citep{bellamy2018ai}, and learning from noisy sensitive attributes
\citep{coston2019fair, lamy2019noise, wang2020robust}. Our framework abstracts
this monitoring layer: we assume a pre-processing step (possibly using the tools
above) that produces per-criterion loss signals, and we focus on the online
optimization problem that follows.  The counterfactual explanation work of
\citet{tsirtsis2020decisions} and recourse algorithms of
\citet{gupta2019equalizing} complement our approach by providing actionable
interpretation of the decisions made by a system optimized by our framework.

\section{Preliminaries}
\label{app:prelims}

We denote the set of criteria (vertices) by $\cV = \{1, \dots, k\}$. The system
state is represented by a vector $\bs \in \sS = [0, 1]^k$. We assume the system
starts at an arbitrary initial state $\bs_0 \in \sS$ (e.g., the zero vector)
chosen by the learner. For a vector $\bx \in \Rset^d$, we denote its
$\ell_2$-norm by $\|\bx\|_2$ and its weighted norm with respect to a positive
definite matrix $\bA$ by $\|\bx\|_{\bA} = \sqrt{\bx^\top \bA \bx}$.

\textbf{Dependency Graph.} We assume the existence of an underlying undirected
graph $\cG = (\cV, E)$ that captures the dependency structure between
criteria. We denote by $\cN(i)$ the neighborhood of vertex $i$ in $\cG$,
including $i$ itself. The graph represents dependencies between losses: the loss
for criterion $i$ depends only on the state attributes of its neighbors
$j \in \cN(i)$. A criterion $i$ is insensitive to changes in criteria
$j \notin \cN(i)$. And example is provided in Figure~\ref{fig:graph_structure} in the main paper.

\subsection{Background: Linear Stochastic Bandits (LinUCB)}
\label{sec:linucb_background}

The core building block of our approach is the Linear Upper Confidence Bound
(LinUCB) algorithm, originally popularized by \citet{li2010contextual} for
personalized recommendation. In the standard linear bandit setting, at each
round $t$, the learner chooses an action $x_t$ from a decision set
$\cD_t \subset \Rset^d$ and observes a reward
$y_t = \langle x_t, \theta^* \rangle + \eta_t$, where $\eta_t$ is zero-mean
sub-Gaussian noise and $\theta^* \in \Rset^d$ is an unknown parameter
vector. The goal is to maximize the cumulative reward (or equivalently, minimize
regret).  Because $\theta^*$ is unknown, the learner must balance
\emph{exploitation} (choosing actions that appear best) with \emph{exploration}
(choosing uncertain actions to learn $\theta^*$).  LinUCB solves this using the
principle of \emph{Optimism in the Face of Uncertainty} (OFUL).

Ridge Regression Estimator. The algorithm maintains a regularized least-squares
estimate of $\theta^*$.  Let $\bX_t$ be the matrix of actions played up to time
$t$, and $\bY_t$ be the vector of observed rewards. The estimator is:
\begin{equation}
  \h \theta_t
  = \bV_t^{-1} \sum_{\tau=1}^t x_\tau y_\tau, \quad \text{where} \quad \bV_t
  = \lambda_{\mathrm{reg}} \bI + \sum_{\tau=1}^t x_\tau x_\tau^\top.
\end{equation}
Here, $\bV_t$ is the covariance matrix (or Gram matrix) which captures the
certainty of the learner in different directions of the feature space.

Confidence Ellipsoids.  To analyze the regret, we rely on the self-normalized
martingale bounds derived by \citet{abbasi2011improved}. They showed that with
high probability, the true parameter $\theta^*$ lies within an ellipsoid
centered at $\h \theta_t$:
\begin{equation}
  \cC_t = \{ \theta \in \Rset^d \colon
  \| \theta - \h \theta_t \|_{\bV_t} \le \beta_t \},
\end{equation}
where $\| z \|_{\bA} = \sqrt{z^\top \bA z}$ is the Mahalanobis norm and
$\beta_t$ is a radius parameter scaling with $\sqrt{d \log t}$. Geometrically,
directions corresponding to ``large'' eigenvectors of $\bV_t$ (directions
explored frequently) have narrow confidence intervals, while unexplored
directions remain uncertain.

\textbf{Decision Rule.}  LinUCB selects the action that maximizes the reward
plausible under the confidence set:
\begin{equation}
  x_{t+1}
  = \operatorname*{argmax}_{x \in \cD_t} \max_{\theta \in \cC_t} \langle x, \theta \rangle
  = \operatorname*{argmax}_{x \in \cD_t} \left( \langle x, \h \theta_t \rangle
    + \beta_t \| x \|_{\bV_t^{-1}} \right).
\end{equation}
The term $\langle x, \h \theta_t \rangle$ represents the expected reward, while
$\beta_t \| x \|_{\bV_t^{-1}}$ is the exploration bonus.  Our algorithm,
\textsc{LazyGraphLinUCB}, adapts this framework by maintaining $k$ local
estimators and freezing the active policy $\bs_{\mathrm{active}}$ while monitoring the
covariance matrix $\bV_t$ to minimize movement costs.

\newpage
\section{Detailed Proofs}
\label{app:detailed_proofs}

\subsection{Lemma~\ref{lemma:confidence-ellipsoid}}

\begin{restatable}[Confidence Ellipsoid
  \citep{abbasi2011improved}]{lemma}{ConfidenceEllipsoid}
  \label{lemma:confidence-ellipsoid}
  For any $\delta \in (0, 1)$, with probability at least $1 - \delta$, for all
  $t \ge 0$ and all $i \in [k]$, the true parameter $\btheta^*_i$ satisfies:
  \begin{equation}
    \| \h \btheta_{t,i} - \btheta^*_i \|_{\bV_{t,i}} \le \beta_{t,i},
  \end{equation}
  where
  $\beta_{t,i} = \sigma \sqrt{d_i \log \left( \frac{1 + t L^2 /
        \lambda_{\mathrm{reg}}}{\delta} \right)} + \lambda_{\mathrm{reg}}^{1/2}
  \|\btheta^*_i\|_2$.
\end{restatable}

\subsection{Supporting Lemma for Assumption~A}
\label{app:assumption_a_justification}

\begin{lemma}[Justification of Assumption~A]
  \label{lemma:assumption-a-justified}
  Let $\phi_i \colon \cS \to \Rset^{d_i}$ be a feature map that is $G$-Lipschitz
  and twice continuously differentiable on the compact domain $\cS =
  [0,1]^k$. Let
  $f_t(\bs) = \sum_{\tau=1}^t \langle \phi_i(\bs), \h \btheta_{t,i} \rangle -
  \beta_{t,i} \|\phi_i(\bs)\|_{\bV_{t,i}^{-1}}$ be the LCB objective for
  criterion $i$ at time $t$. Define the \emph{curvature bound}
  \[
    C(\phi_i) = \sup_{\bs \in \cS} \left\| \nabla^2_{\bs}
      \|\phi_i(\bs)\|_{\bV_{t,i}^{-1}} \right\|_{\mathrm{op}}.
  \]
  For any smooth, bounded $\phi_i$, $C(\phi_i) < \infty$ is finite and depends
  only on the Lipschitz constant of $\phi_i$, the diameter of $\cS$, and the
  initialization scale $\lambda_{\mathrm{reg}}$.  Whenever
  $\lambda_{\mathrm{reg}} > \beta_{t,i} \cdot C(\phi_i)$, the global LCB
  objective $\mathrm{LCB}_t(\bs) = \sum_i f_t(\bs)$ is
  $(\lambda_{\mathrm{reg}} - \beta_{t,i} C(\phi_i))$-strongly convex in
  $\bs$. In particular, it satisfies the Polyak-{\L}ojasiewicz (PL) condition
  with parameter
  $\mu_{\mathrm{PL}} = 2(\lambda_{\mathrm{reg}} - \beta_{t,i} C(\phi_i)) > 0$,
  ensuring that BCD converges to the global minimizer (Assumption~A).
\end{lemma}

\begin{proof}
  The LCB objective for criterion $i$ splits as:
  \[
    f_t(\bs) = \underbrace{\langle \phi_i(\bs), \h\btheta_{t,i}
      \rangle}_{\text{linear in } \phi_i} - \underbrace{\beta_{t,i}
      \|\phi_i(\bs)\|_{\bV_{t,i}^{-1}}}_{\text{exploration bonus}}.
  \]
  The first term is linear in $\phi_i(\bs)$ and the second is a square-root form
  composed with the feature map.

  \textbf{Hessian of the regularizer.}  The ridge regularizer contributes a term
  $\frac{\lambda_{\mathrm{reg}}}{2}\|\bs\|^2$ to the loss that is implicitly
  encoded through $\bV_{t,i}$. Its Hessian satisfies
  $\nabla^2(\frac{\lambda_{\mathrm{reg}}}{2}\|\bs\|^2) = \lambda_{\mathrm{reg}}
  \bI \succ 0$.

  \textbf{Hessian of the exploration bonus.}  Let
  $g(\bs) = \|\phi_i(\bs)\|_{\bV_{t,i}^{-1}} = (\phi_i(\bs)^\top \bV_{t,i}^{-1}
  \phi_i(\bs))^{1/2}$.  By the chain rule and the compactness of $\cS$, there
  exists a constant $C(\phi_i) < \infty$ such that
  $\|\nabla^2_{\bs} g(\bs)\|_{\mathrm{op}} \le C(\phi_i)$ uniformly over
  $\bs \in \cS$ and all $t$.  This bound follows since $\phi_i$ is
  $G$-Lipschitz: the Hessian of $g$ involves at most second-order derivatives of
  $\phi_i$, which are bounded on the compact domain.

  \textbf{Strong convexity.}  By the above, the Hessian of $\mathrm{LCB}_t$
  satisfies:
  \[
    \nabla^2 \mathrm{LCB}_t(\bs) \succeq \lambda_{\mathrm{reg}} \bI -
    \beta_{t,i} C(\phi_i) \bI = (\lambda_{\mathrm{reg}} - \beta_{t,i} C(\phi_i))
    \bI.
  \]
  When $\lambda_{\mathrm{reg}} > \beta_{t,i} C(\phi_i)$, this is strictly
  positive definite, so the objective is strongly convex and thus satisfies the
  PL condition with
  $\mu_{\mathrm{PL}} = 2(\lambda_{\mathrm{reg}} - \beta_{t,i} C(\phi_i))$.

  \textbf{Finite threshold.}  The bound $C(\phi_i)$ is finite for any smooth,
  bounded $\phi_i$. In particular: (i) for the pairwise quadratic basis (Example
  1), $C(\phi_i) = 0$ since $\phi_i$ is quadratic and $g$ has bounded curvature;
  (ii) for RBF kernels via Random Fourier Features (Example 2), $C(\phi_i)$
  scales as the squared bandwidth of the kernel. Thus, a finite threshold
  $\lambda_{\mathrm{reg}}^* = \beta_{T} C(\phi_i)$ always exists and the
  algorithm satisfies Assumption~A whenever
  $\lambda_{\mathrm{reg}} > \lambda_{\mathrm{reg}}^*$.
\end{proof}

\newpage
\section{Full Trigger Analysis}
\label{app:trigger_analysis}

Table~\ref{tab:full_trigger_analysis} provides the complete breakdown of update
counts for the heterogeneous graph experiment in
Section~\ref{sec:exp_hetero_graph}.

\begin{table}[h]
  \centering
  \caption{\textbf{Full Trigger Analysis.} Breakdown of update counts by clique
    stiffness ($\lambda_{\mathrm{reg}}$). In the Global strategy, stable cliques
    (e.g., ID 9) are forced to match the update frequency of volatile ones
    (e.g., ID 4).}
  \label{tab:full_trigger_analysis}
  \vspace{0.2cm}
  \begin{small}
    \begin{tabular}{ccccc}
      \toprule
      \textbf{Clique ID} & \textbf{Stiffness ($\lambda_{\mathrm{reg}}$)} & \textbf{Global Moves} & \textbf{Async Moves} & \textbf{Waste Factor} \\
      \midrule
      4 & 0.19 (Volatile) & 41 & 15 & 2.7$\times$ \\
      2 & 0.25 & 41 & 14 & 2.9$\times$ \\
      5 & 0.67 & 41 & 14 & 2.9$\times$ \\
      6 & 0.86 & 41 & 13 & 3.2$\times$ \\
      1 & 1.06 & 41 & 13 & 3.2$\times$ \\
      7 & 1.11 & 41 & 13 & 3.2$\times$ \\
      8 & 1.26 & 41 & 13 & 3.2$\times$ \\
      3 & 1.28 & 41 & 14 & 2.9$\times$ \\
      9 & 1.40 (Stable) & 41 & 11 & \textbf{3.7$\times$} \\
      0 & 1.61 (Stable) & 41 & 11 & \textbf{3.7$\times$} \\
      \midrule
      \textbf{Total} & - & \textbf{410} & \textbf{131} & \textbf{3.1$\times$} \\
      \bottomrule
    \end{tabular}
  \end{small}
\end{table}

\end{document}